\documentclass[11pt,hidelinks]{article} 
\def\shownotes{1}

\usepackage[backend=biber,
            natbib=true,
            maxbibnames=99,
            minalphanames=3,
            uniquename=false,
            uniquelist=false,            
            maxcitenames=2,
            sorting=ynt,
            giveninits=true,                        
            labeldateparts=true,
            sortcites=true,
            backref=false]{biblatex}
\DefineBibliographyStrings{english}{%
  backrefpage = {page},%
  backrefpages = {pages},%
}

\usepackage[left=1.in, top=1in, bottom=1in, right=1.in]{geometry}
\usepackage[us,12hr]{datetime} %
\usepackage{graphicx, subcaption}
\usepackage{multirow, makecell}
\usepackage[T1]{fontenc}
\usepackage{microtype}
\usepackage{sidecap}
\usepackage{amsmath,amsthm,amsfonts,amssymb}
\usepackage{macro}

\usepackage{graphicx}
\usepackage{booktabs} %
\usepackage{enumitem}
\usepackage{hyperref}
\usepackage{cleveref}
\usepackage{wrapfig}
\usepackage{breakcites}

\usepackage{relsize}
\usepackage{authblk}
\usepackage[symbol]{footmisc}
\usepackage{algorithm}
\usepackage{algpseudocode}

\title{Scalable Multi-Objective and Meta Reinforcement Learning via Gradient Estimation}
\author{Zhenshuo Zhang} 
\author{Minxuan Duan}
\author{Youran Ye}
\author{Hongyang R. Zhang}
\affil{
    Northeastern University, Boston, Massachusetts\\
    Email: \texttt{\{zhang.zhens, duan.mi, ye.you, ho.zhang\}@northeastern.edu}
}

\begin{document}

\maketitle
\begin{abstract}
We study the problem of efficiently estimating policies that simultaneously optimize multiple objectives in reinforcement learning (RL). Given $n$ objectives (or tasks), we seek the optimal partition of these objectives into $k \ll n$ groups, where each group comprises related objectives that can be trained together. This problem arises in applications such as robotics, control, and preference optimization in language models, where learning a single policy for all $n$ objectives is suboptimal as $n$ grows. We introduce a two-stage procedure---meta-training followed by fine-tuning---to address this problem. We first learn a meta-policy for all objectives using multitask learning. Then, we adapt the meta-policy to multiple randomly sampled subsets of objectives. The adaptation step leverages a first-order approximation property of well-trained policy networks, which is empirically verified to be accurate within a $2\%$ error margin across various RL environments. The resulting algorithm, PolicyGradEx, efficiently estimates an aggregate task-affinity score matrix given a policy evaluation algorithm. Based on the estimated affinity score matrix, we cluster the $n$ objectives into $k$ groups by maximizing the intra-cluster affinity scores. Experiments on three robotic control and the Meta-World benchmarks demonstrate that our approach outperforms state-of-the-art baselines by $16\%$ on average, while delivering up to $26\times$ faster speedup relative to performing full training to obtain the clusters. Ablation studies validate each component of our approach. For instance, compared with random grouping and gradient-similarity-based grouping, our loss-based clustering yields an improvement of $19\%$. Finally, we analyze the generalization error of policy networks by measuring the Hessian trace of the loss surface, which gives non-vacuous measures relative to the observed generalization errors.
\end{abstract}

\section{Introduction}

Reinforcement learning (RL) is a technique for sequential decision-making in interactive environments, enabling agents to learn from feedback signals.
A central challenge in RL is to develop agents that generalize across a wide range of tasks rather than solving a single task in isolation.
For instance, a robot needs to master multiple related skills that share an underlying structure~\cite{yu2020metaworld}.
Prior work has studied this problem for two closely related settings.
The first, multitask reinforcement learning, seeks to find a single policy that performs well across a given set of objectives (or tasks)~\cite{sodhani2021multi,joshi2025benchmarking}.
The second, meta-reinforcement learning, aims to learn a meta-initialization that can rapidly adapt to unseen tasks~\cite{finn2017model, rakelly2019efficient}.
In both settings, however, the computational cost of evaluating and adapting a shared policy over many competing objectives grows quickly with $n$.

\begin{figure*}
    \centering
    \includegraphics[width=0.9750\textwidth]{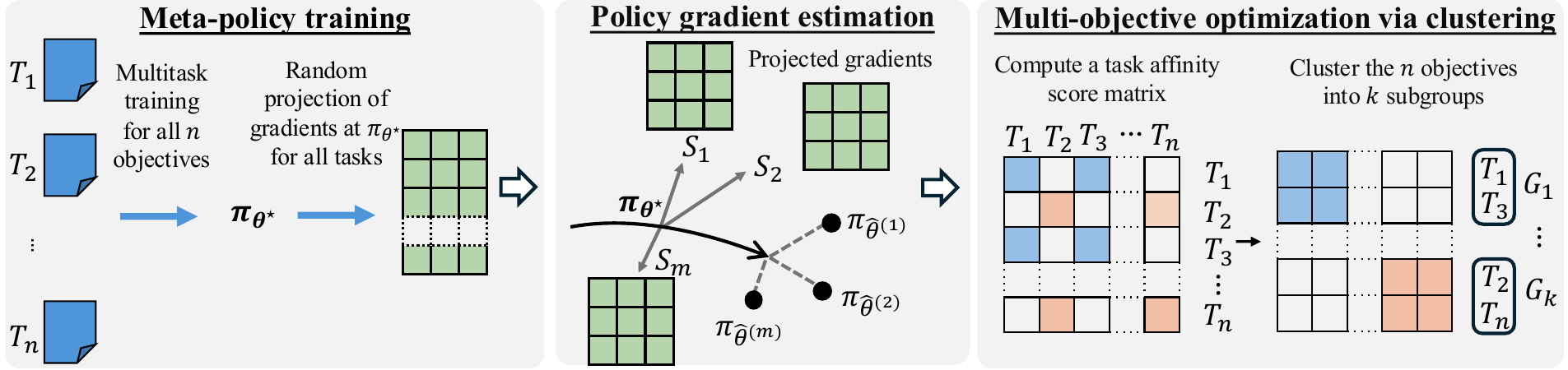}
    \caption{An overview of our approach. \emph{Left:} Run multitask training on all the tasks, $T_1, T_2, \dots, T_n$, and obtain to a meta-initialization policy $\pi_{\theta^\star}$. Store the projected gradients of a surrogate loss with meta-policy $\theta^\star$ for every transition from $t = 1, 2, \dots, N$. \emph{Middle:} Estimate policy adaptation performance on $m$ task subsets using projected gradients as features in logistic regression. \emph{Right:} Compute an $n\times n$ task affinity score matrix based on the estimated loss values. Lastly, run a clustering algorithm to group similar objectives, resulting in $k$ subgroups $G_1, \dots, G_k$, each of which share the same policy within group.}\label{fig_pipeline}
\end{figure*}

This paper studies the problem of designing a (scalable) multi-objective optimizer for reinforcement learning, with extensions to policy meta-adaptation.
Learning from a diverse set of tasks has been shown to improve generalization~\cite{thrun1998learning,cobbe2019quantifying, yin2019meta}.
However, the main difficulty is modeling task relationships and how different tasks transfer information to one another when trained together in a complex network \cite{wuunderstanding}.
Existing approaches have considered measuring the similarity between gradient vectors across tasks to quantify task relationships~\cite{yu2020pcgrad}.
Notably, this gradient-similarity-based measure corresponds to pairwise task-affinity scores, whereas in practice it is natural to model how multiple policies interact when trained together \cite{li2023identification}.
Exhaustively searching over all subsets of $\set{1,2,\dots,n}$ is computationally infeasible: evaluating each subset typically requires a full training procedure, and the number of subsets grows exponentially in $n$~\cite{caruana1997multitask}.
Greedy stepwise selection requires evaluation on $O(n^2)$ subsets, which is again impractical for large $n$ \cite{li2023boosting,li2024scalable}.

To overcome these challenges, we introduce a scalable algorithm for estimating policy performance on arbitrary task subsets without full training. Our algorithm starts by learning a single meta-policy across all tasks, then uses a first-order surrogate model—derived from a Taylor approximation around the meta-policy to efficiently estimate the outcome of fine-tuning on any subset. This approach is analogous to a random ensemble \cite{zhang2025linear}, which allows us to compute a pairwise task affinity matrix and partition the $n$ objectives into $k$ subgroups using a convex clustering procedure. The resulting groups can be trained separately using any multitask or meta-RL optimizers. This overall approach is illustrated in Figure~\ref{fig_pipeline}.

We validate our approach through extensive experiments on both Meta-World~\cite{yu2020metaworld} and robotic control benchmarks~\cite{towers2024gymnasium}.
We find that the first-order approximation applied to the outputs of a policy network yields accurate estimates---typically within $2\%$ error margin---while being up to $26\times$ faster than full training.
Our algorithm consistently outperforms multitask and meta-RL baselines~\cite{yang2020multi,sodhani2021multi,sun2022paco} by $16\%$ on average, and surpasses baseline grouping strategies (such as random grouping or gradient-similarity based grouping) by $23\%$ and $16\%$, respectively. These results highlight the effectiveness and scalability of our task-affinity estimation approach.

Lastly, we conduct a theoretical analysis of the generalization errors of policy networks. The main technical approach is to compute (and measure) the trace of the Hessian of the loss surface \cite{ju2022robust,ju2023generalization,zhangnoise}. See Theorem \ref{thm_hessian} for the precise statement. When evaluated on the above RL environments, we find that the Hessian-based generalization bounds match the scale of empirical generalization errors (see Figure \ref{fig_hessian_results}).

\sloppy
In summary, the contributions of this paper are threefold.
(1) We propose a first-order gradient estimation algorithm to estimate policy performance on any given task subset efficiently.
Based on this gradient estimation, we design a scalable algorithm that computes an $n\times n$ task affinity score matrix and partitions $n$ objectives into $k$ clusters via convex relaxation.
(2) We conduct extensive experiments to validate the efficiency of our approach for several multi-objective and meta reinforcement learning benchmarks. The code for reproducing our results can be accessed at: \url{https://github.com/VirtuosoResearch/PolicyGradEx}.
(3) Finally, we show non-vacuous bounds on the generalization error of the loss surface of policy networks via a PAC-Bayes analysis, and report empirical estimates of the Hessian trace.

\section{Problem Setup}

We consider a Markov decision process in which an agent must learn a set of $n$ tasks $\cT=\{T_1,T_2,\ldots,T_n\}$. Each task $T_i$ is modeled as a Markov decision process (MDP) $T_i=(\mathcal{S},\mathcal{A},P,P_0,r_i,\gamma),$
where all tasks share the same state space $\mathcal{S}$, action space $\mathcal{A}$, transition model $P$, and initial-state distribution $P_0$, discount factor $\gamma$, but differ in their reward functions $r_i$. For a policy $\pi_\theta$ with parameters $\theta$, we define its task-specific expected reward as
\[ R_i(\theta)= \mathbb{E}_{\pi_\theta}\!\left[\sum_{t=0}^{\infty}\gamma^{t}\, r_i(s_t,a_t)\right]. \]

In multi-objective RL, the goal is to learn a single policy that performs well across all tasks. A natural aggregate objective is the average reward
\[ R_{\cT}(\theta)=\frac{1}{n}\sum_{i=1}^{n} R_i(\theta). \]
More generally, for any subset $S\subseteq T$, we define
\[ R_S(\theta)=\frac{1}{|S|}\sum_{T_i\in S} R_i(\theta). \]

Pairwise task interactions can be characterized through this notion. For example, $R_{\{i,j\}}(\theta)$ captures the joint performance on tasks $T_i$ and $T_j$, providing a meaningful measure of their affinity. A key challenge in multitask learning is \emph{negative transfer}, where conflicting gradients from two tasks degrade optimization efficiency and final performance \cite{wuunderstanding}. This occurs when
\[ R_{\{i,j\}}(\theta) < \left( R_i(\theta) + R_j(\theta) \right) / 2, \]
indicating that training the pair together is detrimental compared to training each task separately.

Our goal is to model the underlying relationship between the $n$ tasks, which can be used to group similar tasks or select a small representative subset for meta-learning.
However, searching for the optimal partition requires evaluating $R_S(\theta)$ for multiple subsets $S\subseteq\cT$, since the number of possible subsets is $2^n$.
Greedy stepwise selection again requires evaluating $\mathcal{O}(n^2)$ subsets, where each evaluation typically involves a full RL training or adaptation procedure \cite{li2023boosting,li2024scalableft}. This is again intractable for large $n$.

To address this challenge, we introduce \emph{an efficient gradient-estimation algorithm to approximate the loss of any subset trained on a policy network without repeated training}. This enables us to scale up multi-objective RL to a large number of objectives $n$, as we will describe next.

\section{Our Approach}

This section describes our proposed algorithm for multi-objective reinforcement learning.
The algorithm consists of two main stages. We first find a meta-policy $\theta^\star$ trained using all tasks. Then, compute a surrogate model to estimate the adapted performance of $\theta^{\star}$ on a subset of $\set{1, 2, \dots, n}$. This surrogate modeling framework enables efficient estimation of subset combinations without repeated training.
In the second stage, we estimate the fine-tuning performance $R_{S_i}$ on $m$ randomly chosen subsets $S_1, S_2, \dots, S_m$ of $\cT$, for every $1\le i\le m$.
For each subset, we solve a weighted logistic regression problem to estimate the adaptation parameters $\hat{\theta}_{S_i}$ and then estimate an approximate loss.
Importantly, we run this second step using precomputed gradients and functional values at $\theta^{\star}$ from the first stage.
The key technique that enables this gradient estimation is \emph{a first-order approximation property of the policy reward} that we also empirically verify for various RL settings.

\subsection{Surrogate Modeling}

Our goal is to estimate the output of a policy update starting from the meta-initialization $\theta^\star$. To do so efficiently, we linearize the policy gradient objective around $\theta^\star$ and reformulate the one-step update as a weighted logistic regression problem.

We begin with the standard policy gradient objective used in algorithms such as Proximal Policy Optimization (PPO) \cite{schulman2017proximal}:
\[ J(\theta) = \mathbb{E}_t \left[ r_t(\theta) \hat{A}_t \right], \text{ where }
\displaystyle r_t(\theta) = \frac{\pi_{\theta}(a_t|s_t)}{\pi_{\theta^\star}(a_t | s_t)} \]
is the probability ratio and $\hat{A}_t$ is the advantage function for the state-action pair $(s_t, a_t)$ at time $t$.
We perform a first-order Taylor expansion of the log-probability $\log\pi_{\theta}(a_t|s_t)$ around the initial parameters $\theta^\star$. 
Let $\Delta \theta = \theta - \theta^\star$ and define the gradient feature vector \[ g_t = \nabla \log \pi_{\theta}(a_t | s_t)\big|_{\theta = \theta^\star}. \]
The first-order Taylor expansion of the above log-probability-ratio is given by:
\begin{align}
    \log r_t(\theta) &= \log \pi_{\theta}(a_t | s_t) - \log \pi_{\theta^\star}(a_t | s_t) \nonumber \\
    &= g_t^\top \Delta \theta + \epsilon, \label{eq_log_policy_approx}
\end{align}
where $\epsilon$ refers to the approximation error of the expansion.
By applying equation \eqref{eq_log_policy_approx} back to $J(\theta)$, we have:
\begin{align*}
    J(\theta) &\approx \mathbb{E}_t \left[ \hat{A}_t (1 + g_t^\top \Delta \theta) \right].
\end{align*}
Since $\mathbb{E}_t[\hat{A}_t]$ does not depend on $\Delta\theta$, maximizing $J(\theta)$ reduces to maximizing the following which is linear in $\theta$: \[ \mathbb{E}_t[\hat{A}_tg_t^\top (\theta - \theta^{\star})]. \]

\medskip
\noindent\textbf{Surrogate loss.} Next, we turn the above reward maximization on $J(\theta)$ into a weighted binary classification problem.
For every $(s_t, a_t, \hat{A}_t)$, for any $t \in \set{1, 2, \dots, N}$, define:
\begin{itemize}
    \item {Target binary label} $y_t = \text{sign}(\hat{A}_t) \in \{-1, +1\}$, indicating if the action was better or worse than zero.%
    \item {Classifier score} $z_t = g_t^\top \Delta \theta$, linear score for the update.
    \item {Sample weight} $w_t = |\hat{A}_t|$, magnitude of the reward. %
\end{itemize}
This yields a per-sample surrogate loss as follows:
\[\ell(g_t, y_t, w_t;\Delta\theta) = w_t \cdot \log(1 + (-y_t(g_t^\top\Delta\theta))).\]

For a given task subset ${S}$, let $\mathcal{D}_{S}$ denote the set of trajectory samples $(g_{i,t},y_{i,t},w_{i,t})$ from all the tasks $T_i \in {S}$, across all the samples (with an extra index $i$).
The average loss for the subset ${S}$ is defined by the average loss over all the samples from all the tasks within the subset $S$:
\begin{align}
    \hat{L}_S(\theta)&=\frac{1}{{|\mathcal{D}_{S}|}} \sum_{(g,y,w)\in\mathcal{D}_{S}} \ell(g,y,w;\theta). \label{eq_surrogate_model}
\end{align}
Minimizing this loss over $\Delta\theta = \theta - \theta^{\star}$ yields the estimated adaptation for the subset $S$ from the meta-policy $\theta^{\star}$.

\begin{algorithm}[t!]
\raggedright
\caption{Policy gradient estimation (\acronym{})}\label{alg_1}
\textbf{Input:} $m$ random subsets from $\mathcal{T}$, ${S}_1, S_2, \dots, {S}_m$ \\
\textbf{Require:} Initial meta-policy parameter $\theta^\star\in\real^p$; Number of episodes $N$; Projection dimension $d$
\begin{algorithmic}[1]
    \State $\pi_{\theta^\star} \leftarrow$ Setup the meta-policy for all tasks in $\mathcal{T}$
    \State ${S}_1, {S}_2, \dots, {S}_m \leftarrow m$ random subsets of $\mathcal{T}$ with size $\alpha$ 
    \State $P \leftarrow$ A $p$ by $d$ isotropic Gaussian random projection matrix
    \For{$t = 1,\dots, N$}
        \For{$T_i\in\mathcal{T}$}
            \State $g_{i,t} \leftarrow P^\top \nabla \pi_{\theta^\star}(a_t|s_t)$  %
            \State $w_{i,t} \leftarrow |\hat{A}_t|$
            \State $y_{i,t} \leftarrow \text{sign}(\hat{A}_t)$
        \EndFor
    \EndFor
    \For{$j=1,\dots,m$}
        \State $\Delta\hat{\theta}_d\leftarrow\arg\min_{\Delta\theta\in\real^d} \hat{L}_{{S}_j}(\theta^{\star} + P \Delta\theta)$ %
        \State $\hat{\theta}^{(j)}\leftarrow\theta^\star + P \Delta \hat{\theta}_d$
        \State $\hat{f}({S}_j)\leftarrow -\hat{L}_{{S}_j}(\hat{\theta}^{(j)})$ 
    \EndFor
    \State Return ${\hat{f}({S}_1), \dots, \hat f(S_m)}$
\end{algorithmic}
\end{algorithm}

\medskip
\noindent\textbf{Random projection.} Since the gradient vectors $g_{i,t}\in\mathbb{R}^p$ may be high-dimensional (with $p$ in the order of millions), we apply random projections to reduce the dimension down to a few hundred (in practice, $d = 400$ suffices).
This projection provably preserves the pairwise similarity between all pairs of gradient vectors via the Johnson–Lindenstrauss Lemma~\cite{johnson1984extensions}.
We sample a random projection matrix $P\in\mathbb{R}^{p\times d}$, where $d\ll p$ and every entry of $P$ is independently drawn from a Gaussian $\mathcal{N}(0, d^{-1})$.
We project the high-dimensional gradients to a $d$-dimensional space as \[ \tilde{g}_{i,t}=P^\top g_{i,t}. \]
Then, we solve a logistic regression problem in the $d$-dimensional space with $\theta_d \in \mathbb{R}^d$:
\begin{align}
\widehat{\theta}_d \leftarrow \arg\min_{\theta_d \in \mathbb{R}^d} \frac{1}{|\mathcal{D}_{S_j}|}\sum_{(g, y, w) \in \mathcal{D}_{S_j}} \ell(\tilde{g}, y, w; \theta_d),
\end{align}
for any $j = 1, 2, \dots, m$.
Finally, we turn the solution from dimension $d$ back to the original dimension $p$ as
\[ \widehat{\theta}^{(j)} = \theta^\star + P\widehat{\theta}_d. \] 
The complete description of this surrogate modeling procedure is provided in Algorithm~\ref{alg_1}.

\subsection{Evaluation of First-Order Policy Approximation}

A critical assumption is the accuracy of the approximation in equation \eqref{eq_log_policy_approx}.
Here we observe that the approximation error remains negligible for $\theta$ around $\theta^\star$.
We evaluate this approximation by measuring the relative residual sum of squares (RSS) error across multiple RL environments. Specifically, we compute:
\[ \frac{\left( \log \pi_{\theta}(a_t | s_t) -  \log \pi_{\theta^\star}(a_t | s_t) - g_t^\top \Delta\theta \right)^2}{(\log \pi_{\theta}(a_t | s_t))^2}, \]
where $\Delta \theta = \theta - \theta^\star$ and $g_t = \nabla \log \pi_{\theta}(a_t | s_t)\big|_{\theta = \theta^\star}$.

We begin by verifying that adapted policies remain close to the meta-initialization $\theta^\star$.
Using ten randomly sampled subsets, we evaluate the relative distance $\frac{\|\theta-\theta^\star\|_F}{\|\theta^\star\|_F}$, confirming that adaptation stays within a local neighborhood where the approximation is valid.

We report the RSS error measured on MT10 from the MetaWorld benchmark, CartPole, LunarLander, and Highway in Table~\ref{table_motivation_result}.
In these environments, the policy, a four-layer MLP, is trained to perform various control tasks by selecting an action based on its output given the states. We sample $2048$ steps per task to obtain the gradients.
The results are averaged over $10$ randomly sampled subsets of size $5$. For a proper initialization $\theta^\star$, the approximation error is less than $2\%$ when the updated policy $\theta$ remains close to the initial policy. 
This verifies that the first-order expansion provides a sufficiently accurate local model for surrogate-based policy estimation.
The approximation error increases from less than $2\%$ to up to $10\%$ when the policy's parameter distance from initialization grows to around $5\%$, defining the boundaries of our method's applicability.

\begin{table}[t!]
\centering
\caption{We empirically find that the approximation error of $\epsilon$ is negligible for several interactive RL environments. We report the mean and standard deviation of the approximation error across $10$ random subsets. We defer a detailed description of the setup to the experiments section.}\label{table_motivation_result}
{
\begin{tabular}{@{}lcccc@{}}
\toprule
{Distance} & {MT10} & {CartPole} & {Highway} & {Lunarlander}\\
\toprule
$0.1\%$ & $0.01_{\pm0.01}\%$ & $0.12_{\pm0.14}\%$ & $0.02_{\pm0.02}\%$ & $0.06_{\pm0.01}\%$  \\
$0.5\%$ & $0.43_{\pm0.73}\%$ & $0.73_{\pm0.10}\%$ & $0.11_{\pm0.09}\%$ & $0.03_{\pm0.02}\%$\\
$1.0\%$ & $0.32_{\pm0.56}\%$ & $0.98_{\pm0.65}\%$ & $2.04_{\pm0.58}\%$  & $0.48_{\pm0.01}\%$\\
\bottomrule
\end{tabular}}
\end{table}

\subsection{Task Affinity Grouping}

We next quantify the pairwise relationships between tasks by constructing a task affinity matrix $U\in\mathbb{R}^{n\times n}$, where each entry $U_{i,j}$ captures the collaborative behavior between tasks $T_i$ and $T_j$.
Larger values indicate that jointly fine-tuning on the two tasks produces higher surrogate performance.

A naive approach is to repeatedly train a model for all the $\binom{n}{2}$ subsets, which is impractical.
Instead, we sample $m$ random subsets $\{{S}_1,\dots,{S}_m\}$ and use the surrogate model to compute their estimated performance $\hat{f}({S_i})=-\hat{L}_{S_i}(\hat{\theta}_i)$, for all $i = 1, 2, \dots, m$ based on Algorithm \ref{alg_1}.
We then compute an $n\times n$ task affinity score matrix as follows:
\begin{align}
    U_{i,j} = \frac{1}{n_{i, j}} \sum_{1\leq l\leq m: \{T_i, T_j\}\in {S}_l} \hat{f}({S}_l), \label{eq_score}
\end{align}
for all $i$ and $j$ in $\set{1,2,\dots, n}$, where $n_{i,j}$ is the number of sampled subsets containing both $T_i$ and $T_j$.
As a remark, provided that $m = O(n^2)$, then $n_{i, j}$ must be nonzero with high probability, for all $i$ and $j$ in $\set{1, 2, \dots, n}$.

We then cluster tasks by solving a convex optimization problem that maximizes intra-cluster affinity, following a trace-regularized relaxation \cite{awasthi2015relax}. 
This step is very fast since it runs on an $n$ by $n$ matrix with $n$ being at most several hundred, which can typically be solved in just a few seconds.
Moreover, this procedure can be repeated with different values of $\lambda$ to determine the optimal number of clusters $k$ for downstream performance. 
For brevity, we defer a complete description of the convex relaxation program to the Appendix.
The whole procedure is described in Algorithm~\ref{alg_tag}.

\begin{algorithm}[t!]
\caption{Clustering related objectives into subgroups using the task affinity score matrix}\label{alg_tag}
\textbf{Input:} $n$ tasks $\cT$; number of desired clusters $k$ \\
\textbf{Require:} Number of subsets $m$ with size $\alpha$; Regularization parameter $\lambda$; Number of episodes $N$; Projected dimension $d$ \\
\textbf{Output:} A disjoint partition of $\set{1,2,\dots,n}$ into $k$ groups
\begin{algorithmic}[1]
    \State $\theta^\star \leftarrow$ Train a meta-initialization policy with $\mathcal{T}$
    \State ${S}_1, \dots, {S}_m \leftarrow$ Sample $m$ subsets of size $\alpha$ from $\mathcal{T}$
    \State $\hat{f}({S}_1), \hat f(S_2), \dots, \hat f(S_m) \leftarrow$ Apply \acronym{} with $(S_1, S_2, \dots, S_m; \theta^\star; N, d)$
    \State $U \leftarrow$ An $n \times n$ affinity score matrix via equation \eqref{eq_score}
    \State ${X} \leftarrow$ Solve a convex relaxation program \eqref{eq_trace_reg} with $U$ and regularization parameter $\lambda$
    \State $G_1, G_2, \dots, G_k \leftarrow$ Round ${X}$ into $k$ subgroups
\end{algorithmic}
\end{algorithm}

\section{Experiments}\label{sec_exp}

We conduct a comprehensive set of experiments to validate \acronym{} and its application across various multi-objective reinforcement learning benchmarks. Our evaluation is designed to answer two key questions: (1) Does our surrogate model accurately and efficiently approximate the outcome of full policy training on various task subsets?
(2) Do task groups identified by our algorithm lead to superior performance in downstream multitask and meta-RL evaluations compared to state-of-the-art and heuristic baselines?

Our finding shows that our surrogate modeling approach achieves high accuracy in estimating multitask RL performance, with over $0.73$ normalized mutual information relative to the ground-truth, while reducing the number of floating-point operations (FLOPs) by up to $26\times$.
In downstream evaluations, \acronym{} outperforms existing multitask optimizers by $19\%$ in multitask RL benchmarks and achieves a $13\%$ improvement in the meta-RL setting.
Furthermore, we compare our task affinity grouping algorithm to random grouping and gradient-similarity based grouping, showing that \acronym{} achieves a $19\%$ improvement over both.
Ablation analysis validates the use of random projections and the choice of $k$ in the algorithm.

\subsection{Experimental Setup}

\textit{Environments.}
We evaluate our method on two types of benchmarks.
First, we use MT10 from the Meta-World benchmark~\cite{yu2020metaworld}, which consists of 10 diverse robotic manipulation tasks. 
While these tasks share common state and action spaces, their reward functions and dynamics vary (e.g., a goal position in one task may correspond to an object's position in another).
Second, we use three classic control environments from Gymnasium~\cite{towers2024gymnasium} and MO-Gymnasium~\cite{felten_toolkit_2023}: CartPole, Highway, and LunarLander. For each, we generate a distribution of 10 source tasks by altering key physical parameters (e.g., pole length for CartPole, traffic density for Highway, and gravity for LunarLander). %

\textit{Baselines.}
We compare our method against two categories of baselines.
(1) Multi-objective RL baselines: These represent standard and state-of-the-art multitask approaches for multitasks training: a single policy trained on all tasks; Soft Modularization~\cite{yang2020multi}, which employs a routing-based selection mechanism; PaCo~\cite{sun2022paco}, an ensemble method that composes policies from a shared parameter subspace; and CARE~\cite{sodhani2021multi}, which performs inference using contextual information to adapt policies.
(2) Grouping baselines: These methods use the same number of groups as our approach but employ different heuristic grouping procedures. This allows us to isolate the benefit of our affinity modeling. We include random grouping (where tasks are randomly assigned to $k$ groups) and gradient-similarity-based grouping (where tasks are clustered based on the cosine similarity of their respective policy gradients).

\textit{Implementations.}
For the MT10 benchmark, we group the $n = 10$ tasks into $k = 3$ subsets and train a separate policy for each subgroup using soft modularization. We measure performance using the average success rate per task. 

For the control environments, we form the meta-training set by randomly selecting one from each task group provided by \acronym{}. We then use MAML~\cite{finn2017model} to train a meta-policy on the assigned tasks. 

The final performance is measured by the average reward after 200 adaptation steps on 50 unseen target tasks. 
All experiments are conducted using PyTorch on an Ubuntu server with an Intel Xeon E5-2623 CPU and an NVIDIA Quadro RTX 6000 GPU.

\subsection{Results for Loss Function Estimation}

\begin{table}[t!]
\caption{Normalized Mutual Information (NMI) between estimated and actual clusters, measured on two environments trained with a multi-layer perceptron (MLP). In the last column, the speedup is measured as the ratio of the FLOP count between full training and \acronym{}.}\label{table_subset_error}
\centering
{\begin{tabular}{c|cc|c}
\toprule
\# MLP layers           & Meta-World & LunarLander & Speedup \\ \hline
$2$ & $0.76$  &  $0.73$     & $21\times$ \\
$4$ & $0.76$  &  $0.73$     & $24\times$ \\
$8$ & $0.76$  &  $0.73$     & $26\times$ \\
\bottomrule
\end{tabular}}
\end{table}

We first evaluate both the approximation accuracy and computational cost of \acronym{}.
We establish ground-truth subset clusters based on task-affinity scores derived from rewards obtained by training policies on different task subsets. We then use Normalized Mutual Information (NMI) to quantify the accuracy of the subsets estimated by \acronym{} against this ground truth.

As shown in Table~\ref{table_subset_error}, on both  MetaWorld and LunarLander, \acronym{} identifies subset clusters that achieve over $0.73$ similarity to those obtained from full-policy training, while reducing FLOPs by a factor of $26\times$.
By contrast, the NMI under random clustering is approximately $0.2$.
Moreover, the relative error between the estimated task affinity matrix and the ground-truth matrix is at most $0.2$.

\subsection{Results for Multi-Objective Optimization}

We report the evaluation results in Table~\ref{tab_results}. 
On the Meta-World benchmark, our approach improves the average success rate by $21\%$ compared to multitask optimizers.
It also achieves a $62\%$ improvement over training with random groups and a $35\%$ improvement over training with gradient-similarity-based groups.
This observation suggests that random grouping and gradient similarity-based grouping in Meta-World tasks can result in negative transfer. In contrast, \acronym{} selects task groups with positive transferability, thereby improving the final success rate. 

\begin{table*}[t!]
\centering
\caption{Comparison of our approach with several baseline methods. We report the average success rate on the Meta-World benchmark and the rewards (after adaptation) for three robotic control environments. As for the meta-RL setting, we several a subset of tasks for which we set as the source transfer tasks. We note that CARE does not apply to the control tasks in the meta-RL setting. We report the mean and standard deviation from five runs.}\label{tab_results}
{\begin{tabular}{lcccc}
\toprule
RL Environment   & Meta-World & CartPole & Highway & LunarLander \\ \hline
Multi-task training \cite{yu2020metaworld}  &  $71.3_{\pm1.2}\%$ & $145.9_{\pm9.0}$ &  $140.0_{\pm4.6}$ & $53.8_{\pm14.6}$\\
Soft modularization \cite{yang2020multi} & $82.0_{\pm1.1}\%$ & $139.3_{\pm9.5}$ & $141.3_{\pm3.5}$ & $66.1_{\pm13.0}$\\
PaCo \cite{sun2022paco}         &  $73.1_{\pm1.1}\%$   & $144.5_{\pm5.2}$ & $136.6_{\pm6.4}$ &  $62.6_{\pm11.4}$  \\
CARE \cite{sodhani2021multi}         &  $84.0_{\pm1.8}\%$   &  /        &  /       &      /       \\ %
Randomly assign each task into $k$ groups &  $58.2_{\pm6.2}\%$ & $144.1_{\pm10.1}$ & $143.4_{\pm5.4}$ & $73.1_{\pm11.7}$ \\
Gradient-similarity-based grouping \cite{yu2020pcgrad} &  $69.6_{\pm1.9}\%$ & $142.0_{\pm8.2}$  & $135.6_{\pm7.6}$ & $80.8_{\pm6.9}$\\
\hline
\textbf{Algorithm \ref{alg_tag} (This paper)} & $\mathbf{94.0_{\pm2.8}}\%$   & $\mathbf{159.2}_{\pm3.8}$  & $\mathbf{153.5}_{\pm7.8}$ & $\mathbf{82.8}_{\pm6.9}$ \\
\bottomrule
\end{tabular}}
\end{table*}

Next, we evaluate the ability of \acronym{} to select a representative task subset for meta-learning in the three control RL environments. 
Using MAML as the meta-learner, our approach improves the final adapted reward by $7\%$ compared to multitask optimizers.

Finally, we compare with two baseline meta-training settings: training on all available source tasks and training on randomly grouped tasks.
Our approach achieves a $13\%$ improvement over both baselines and a $9\%$ improvement compared to gradient-similarity-based grouping.

\subsection{Generalization Error Measurements via Hessians}

Next, we analyze the generalization behavior.
by evaluating a sharpness measure based on the Hessian trace of the test loss
for the policy network and the policy gradient loss.
A smaller trace value suggests a flatter loss landscape.
We implement Hutchinson's estimator and leverage a faster version of this estimator \cite{meyer2021hutch++} to estimate the Hessian trace of RL models.
We compute the policy gradient loss from the target tasks and measure the difference between the training loss and test loss as the generalization error.
Results for the largest eigenvalue of the Hessian are similar.

First, we evaluate the Hessian trace and the policy generalization error between single-task training, multitask training on ten tasks, and training on a task group selected by \acronym{}. We begin training from the initial policy for $100$ iterations, each with $2048$ steps.

\begin{figure*}[t!]
    \begin{subfigure}[b]{0.6\textwidth}
        \centering
        \begin{subfigure}[b]{0.33\textwidth}
            \centering
            \includegraphics[width=0.99\textwidth]{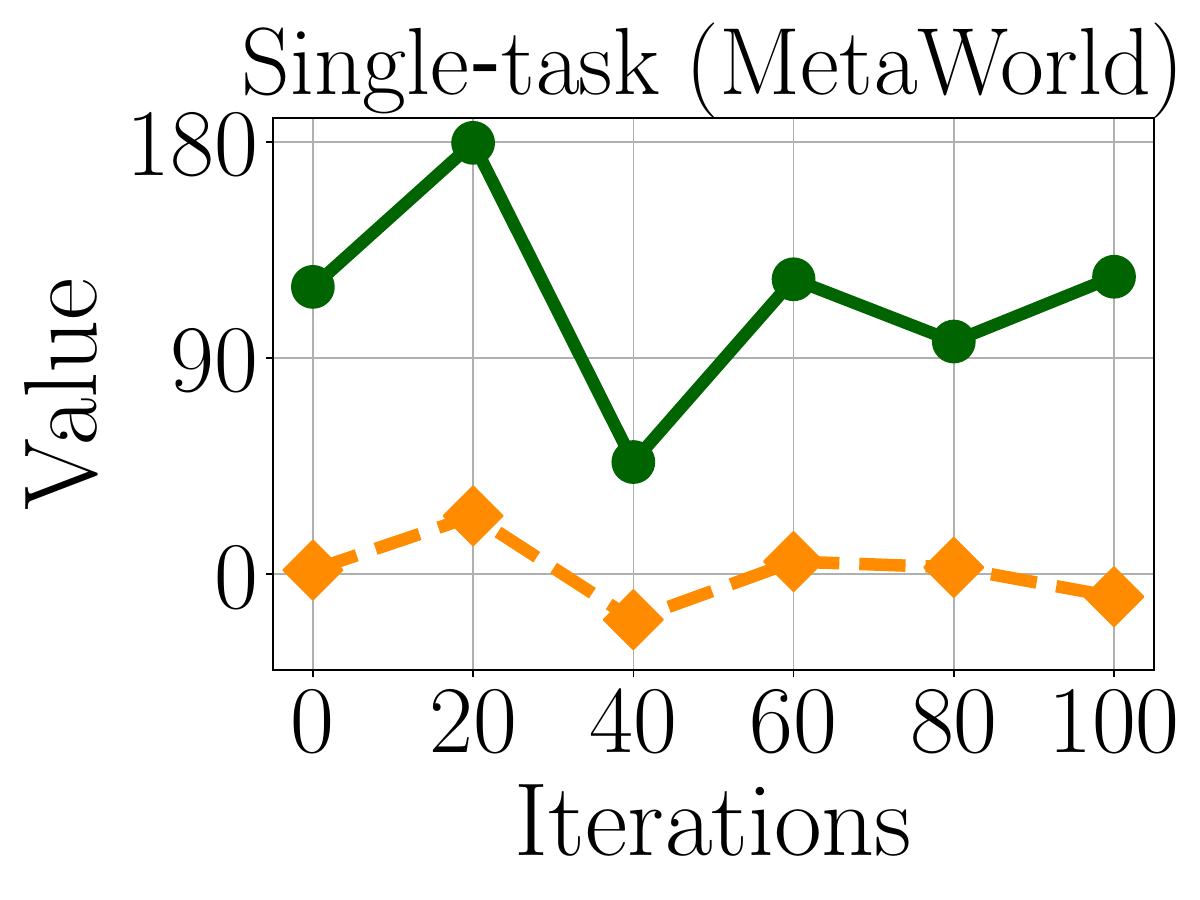}
        \end{subfigure}\hfill
        \begin{subfigure}[b]{0.33\textwidth}
            \centering
            \includegraphics[width=0.99\textwidth]{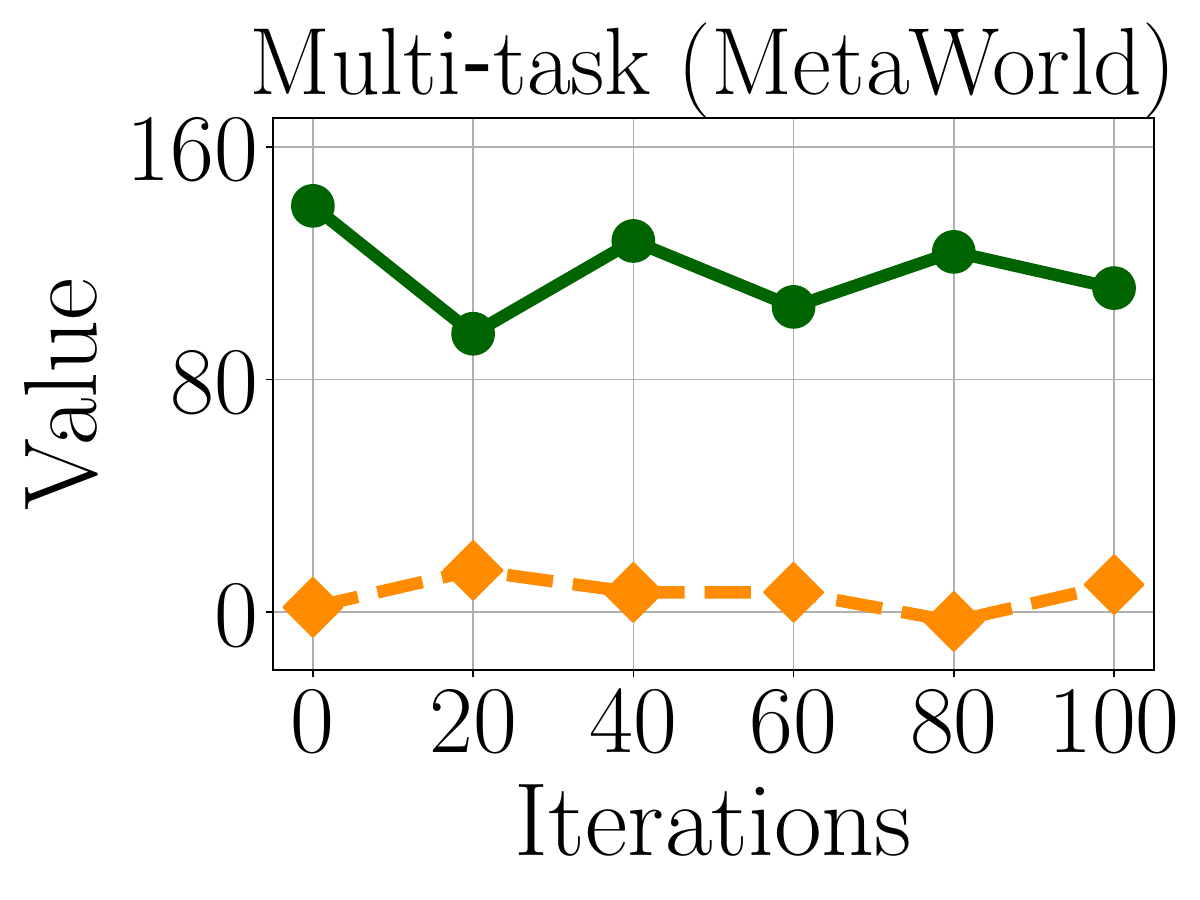}
        \end{subfigure}\hfill
        \begin{subfigure}[b]{0.33\textwidth}
            \centering
            \includegraphics[width=0.99\textwidth]{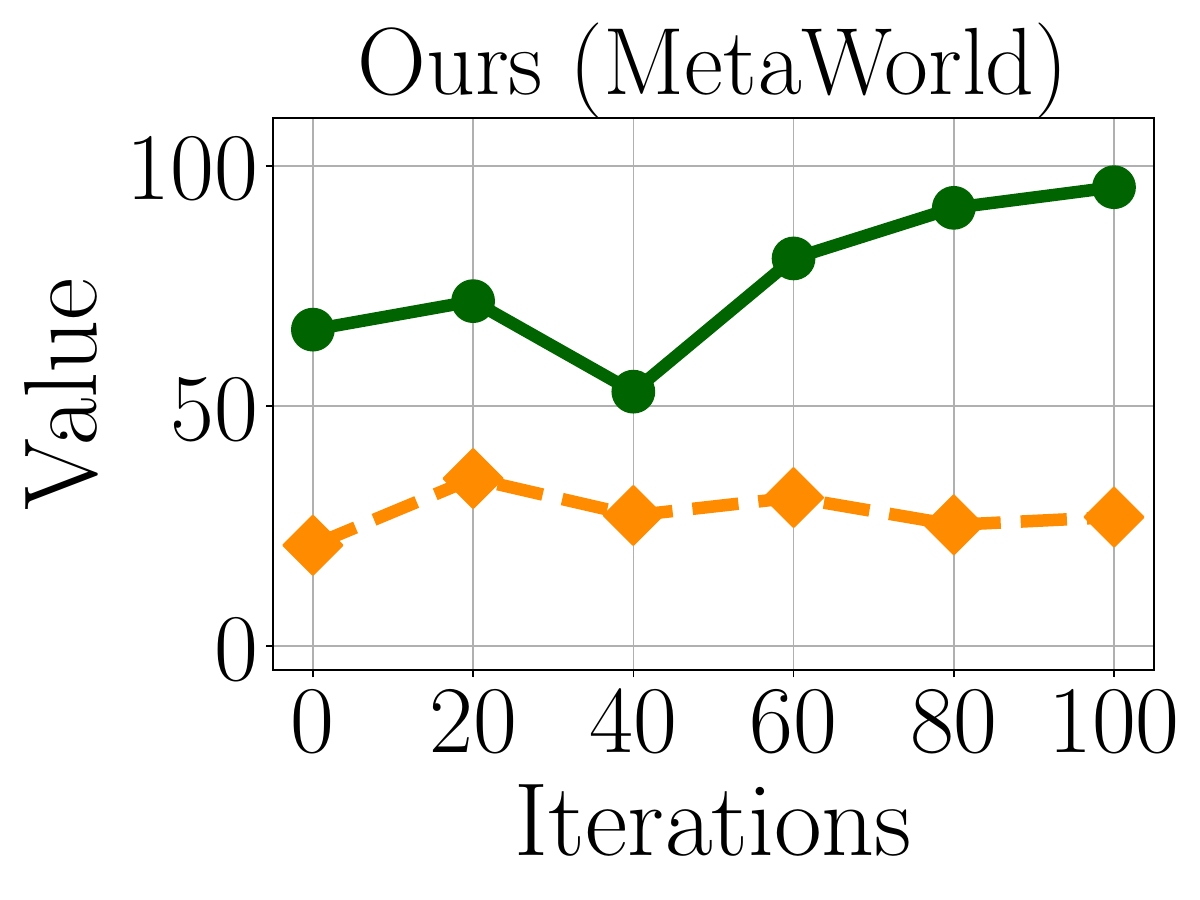}
        \end{subfigure}\hfill \\
        \begin{subfigure}[b]{0.33\textwidth}
            \centering
            \includegraphics[width=0.99\textwidth]{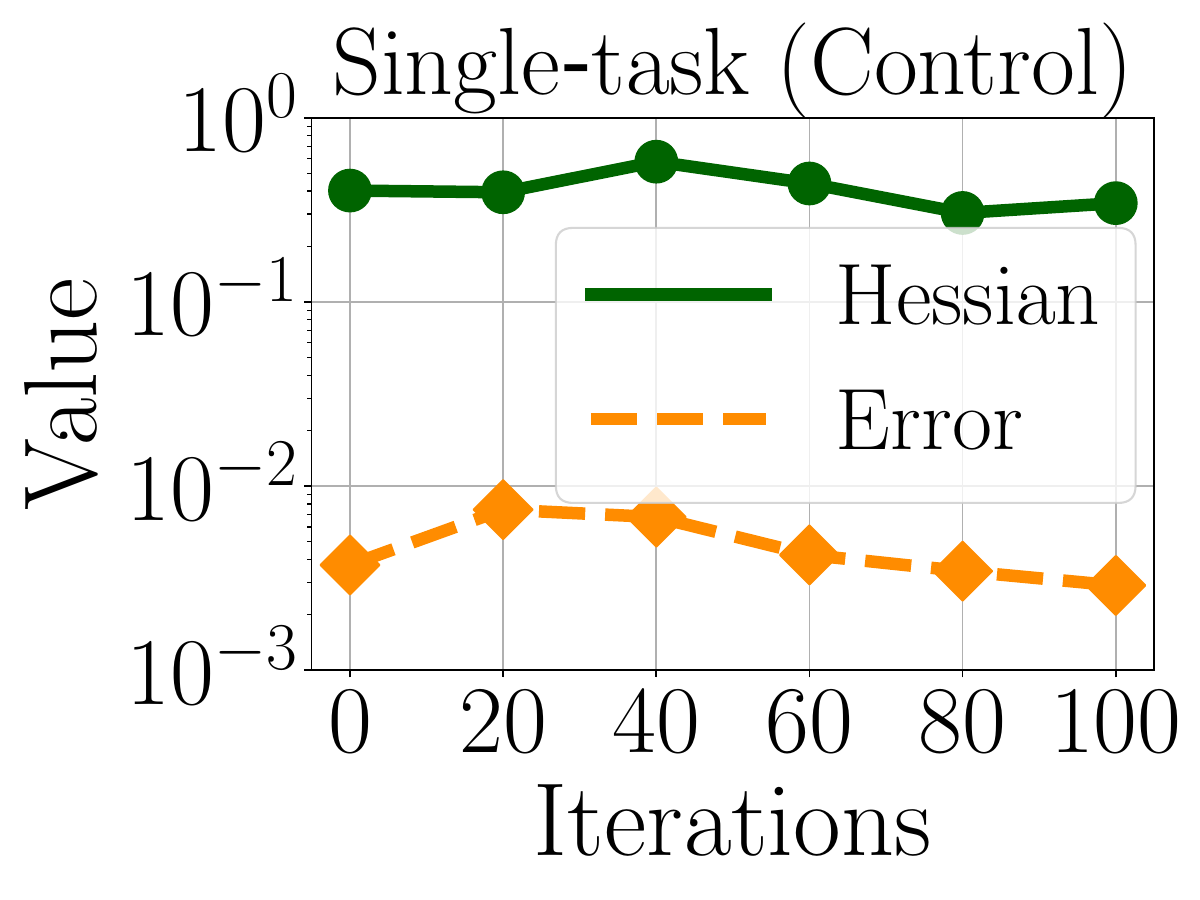}
        \end{subfigure}\hfill
        \begin{subfigure}[b]{0.33\textwidth}
            \centering
            \includegraphics[width=0.99\textwidth]{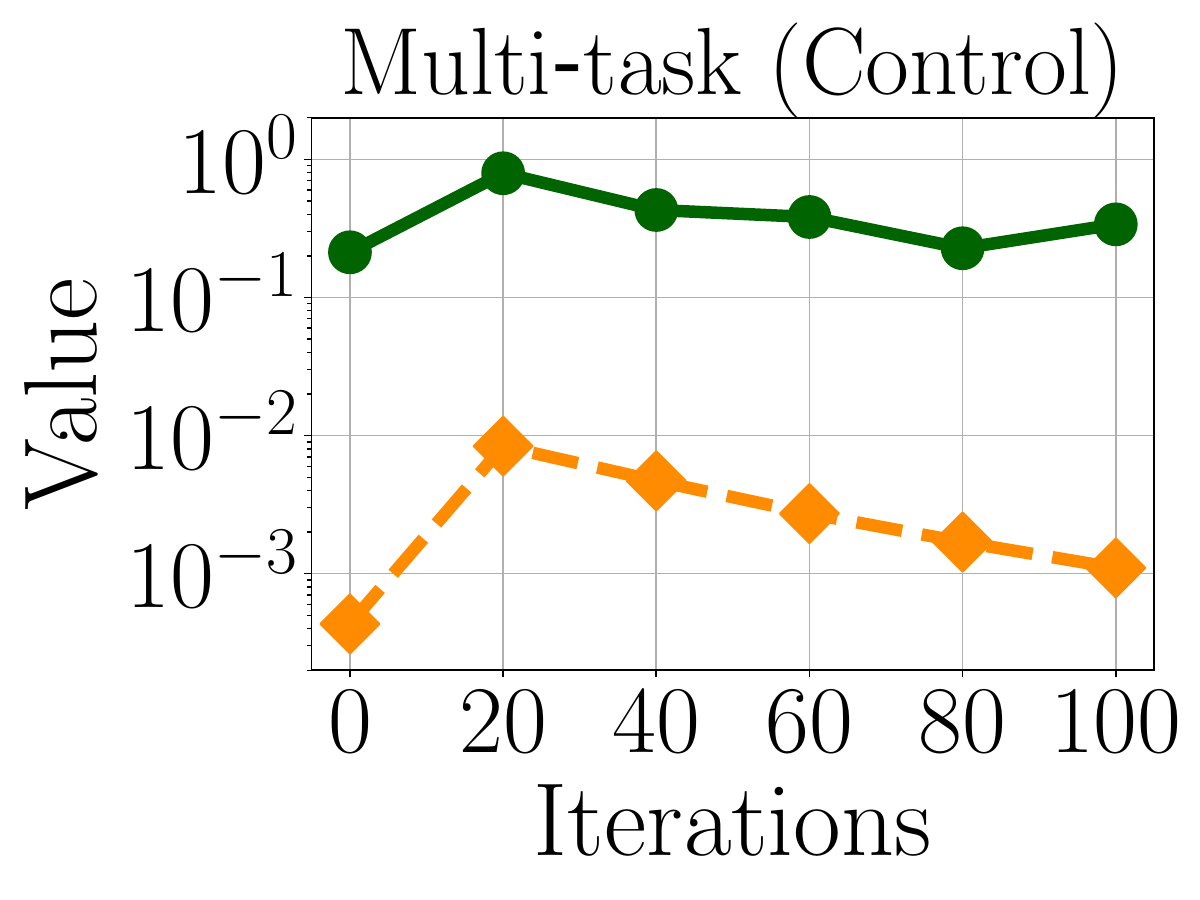}
        \end{subfigure}\hfill
        \begin{subfigure}[b]{0.33\textwidth}
            \centering
            \includegraphics[width=0.99\textwidth]{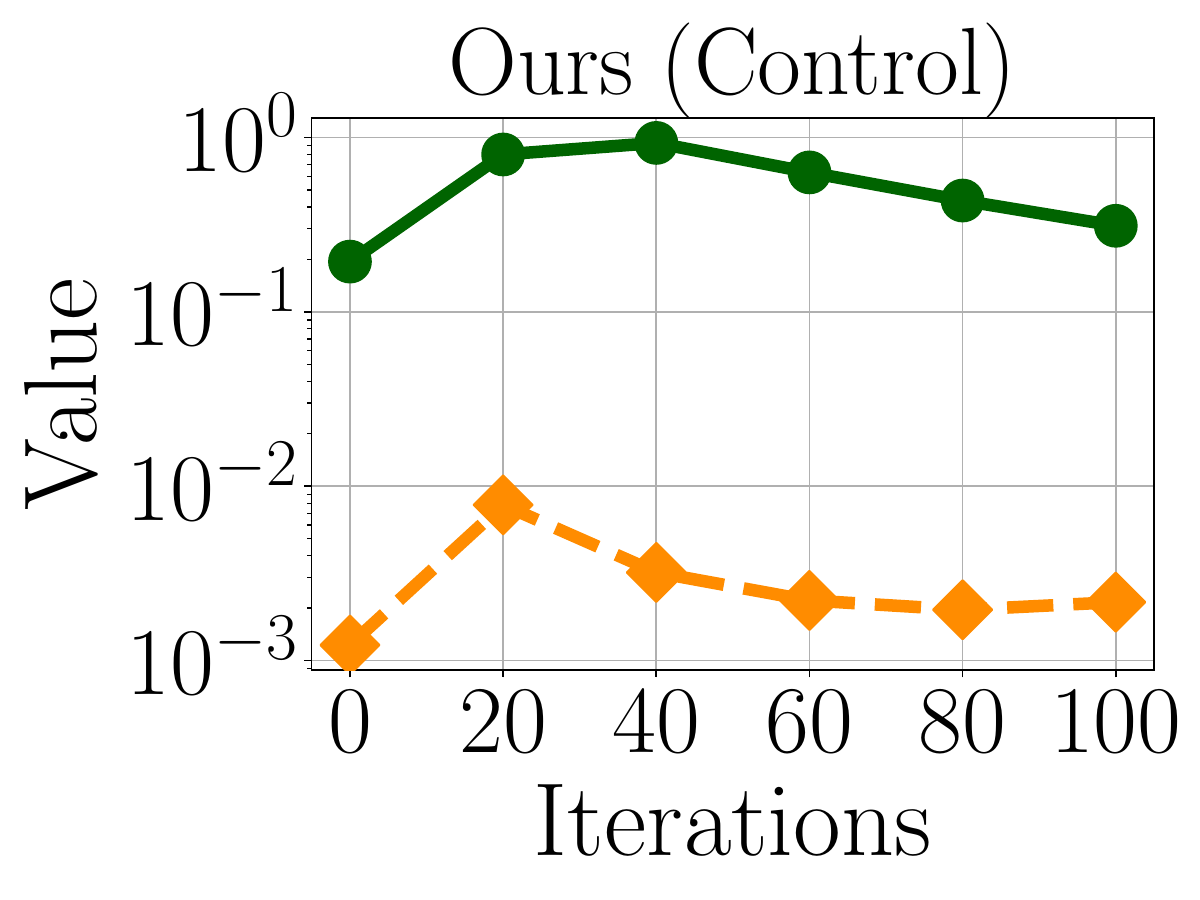}
        \end{subfigure}\hfill 
        \caption{Illustration of Hessian traces and generalization errors}\label{fig_hessian_training_curve}
    \end{subfigure}
    \begin{subfigure}[b]{0.39\textwidth}
        \begin{subfigure}[b]{0.49\textwidth}
            \centering
            \includegraphics[width=0.99\textwidth]{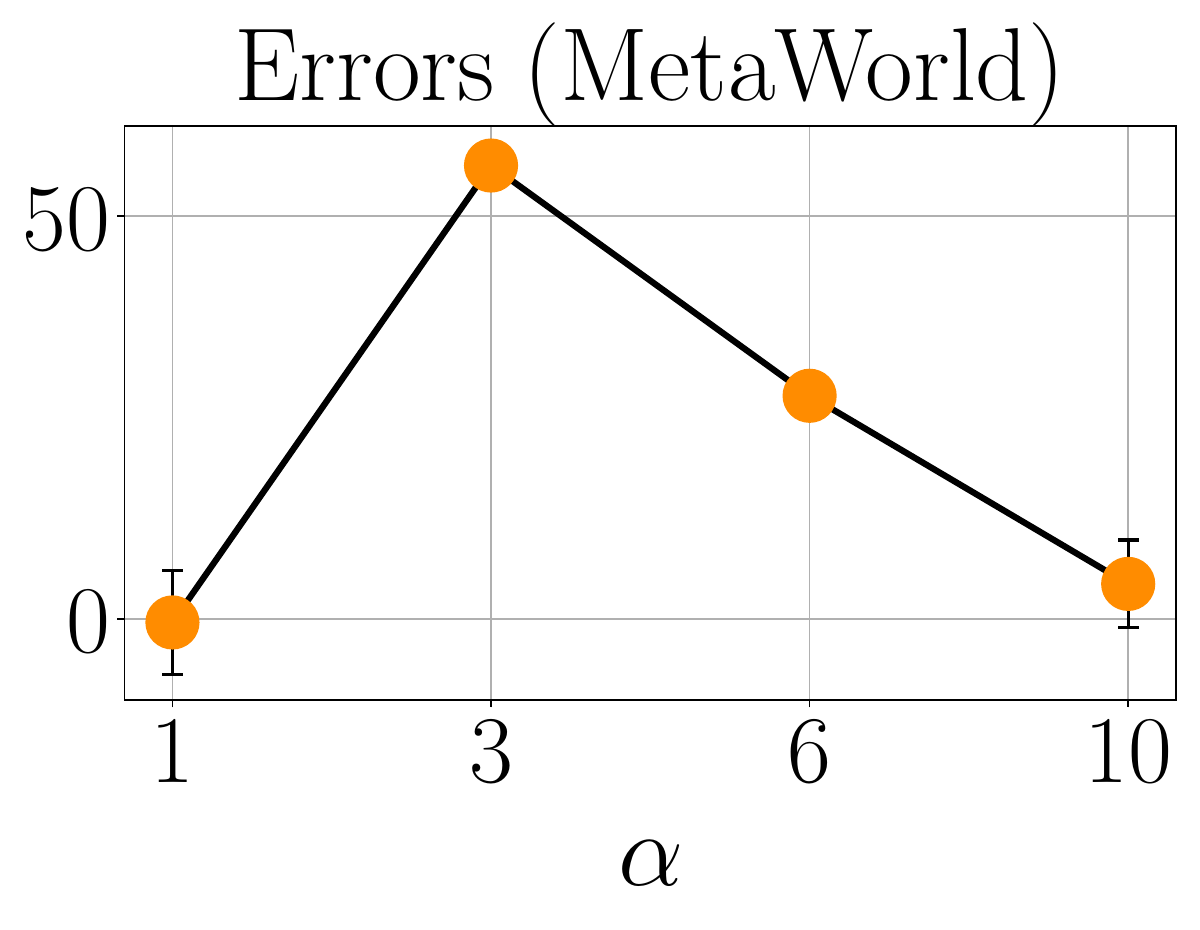}
        \end{subfigure}
        \begin{subfigure}[b]{0.49\textwidth}
            \centering
            \includegraphics[width=0.99\textwidth]{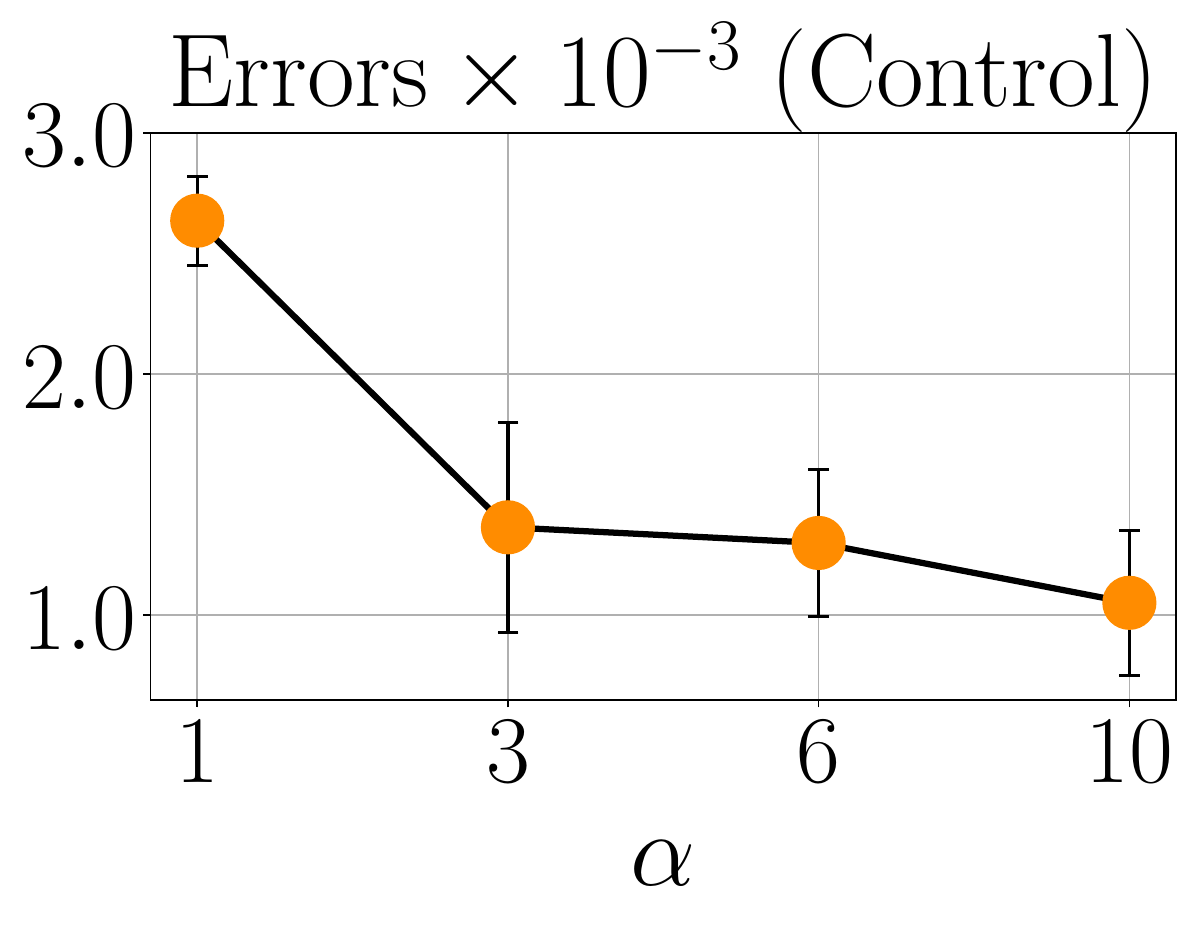}
        \end{subfigure} \\
        \begin{subfigure}[b]{0.49\textwidth}
            \centering
            \includegraphics[width=0.99\textwidth]{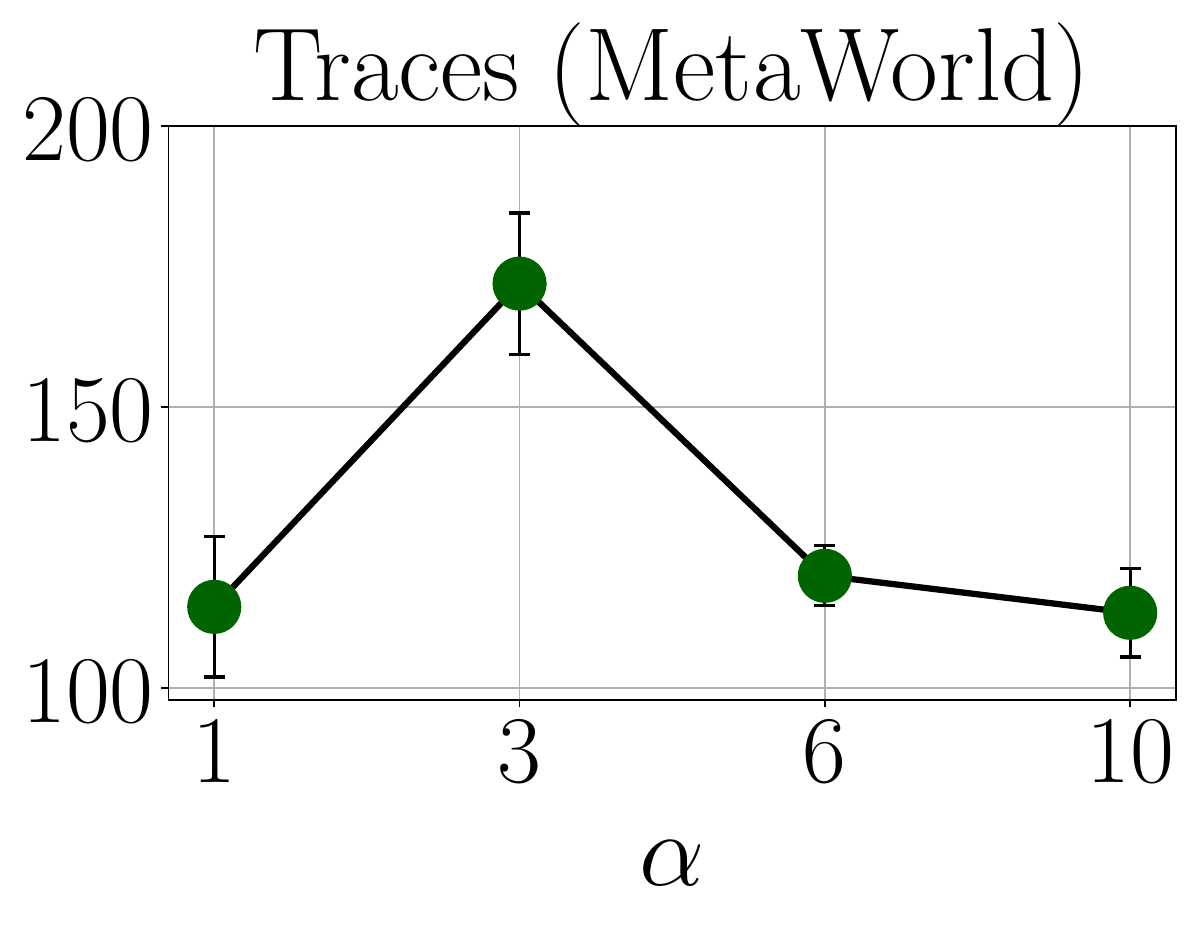}
        \end{subfigure}
        \begin{subfigure}[b]{0.49\textwidth}
            \centering
            \includegraphics[width=0.99\textwidth]{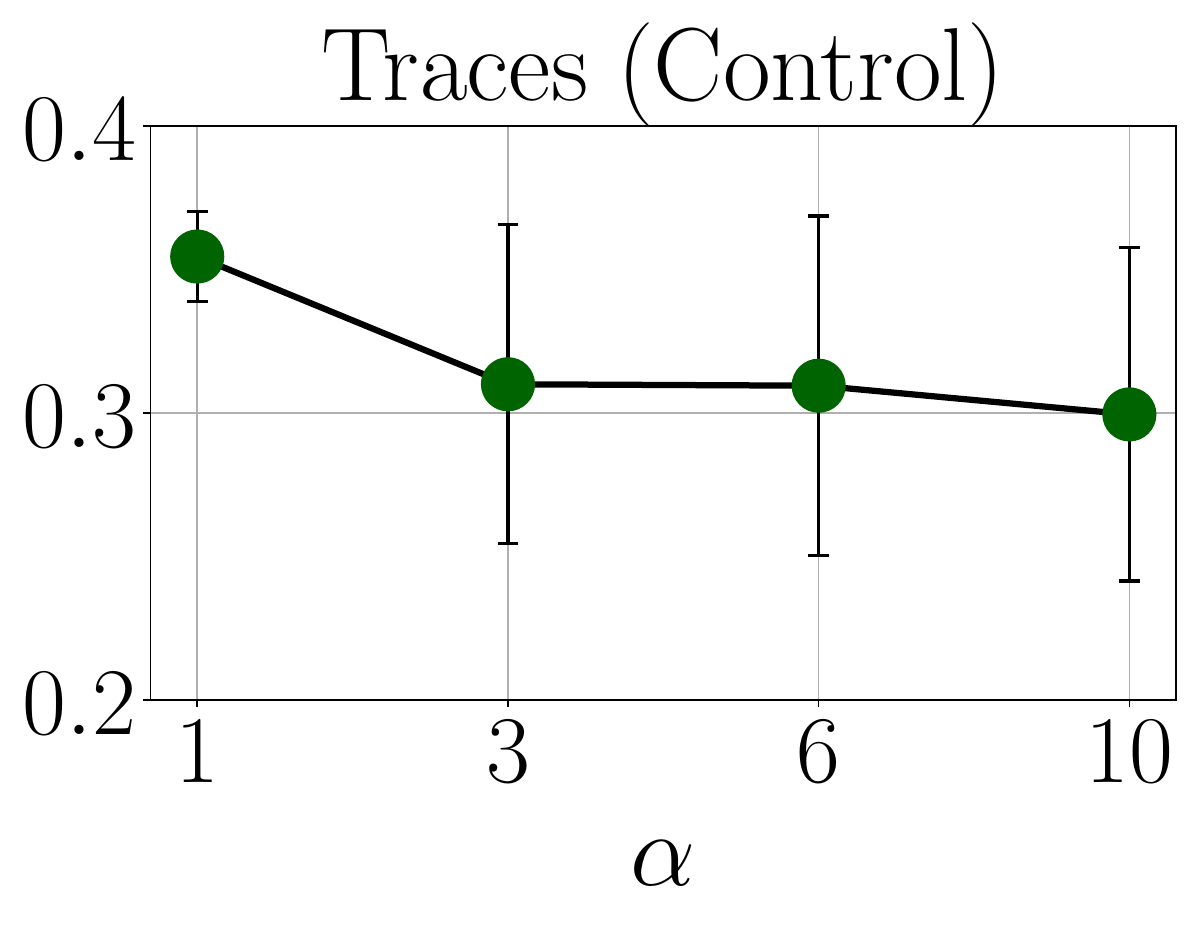}
        \end{subfigure}
        \caption{Varying subset size $\alpha$ in Algorithm \ref{alg_tag}}\label{fig_hessian_vary_k}
    \end{subfigure}
    \caption{Illustrating the Hessian trace measurements and empirical generalization errors with respect to the policy network. Figure \ref{fig_hessian_training_curve}: Showing that the Hessian trace is comparable in scale to the observed generalization errors, tested on Meta-World and a control task. Figure \ref{fig_hessian_vary_k}: Showing that in a Meta-World environment, the generalization error reaches the highest when the subset size $\alpha = 3$, suggesting negative transfer among a small set of tasks.
    In the meta-RL control task, the generalization performance monotonically improves with the addition of more tasks in each group.}\label{fig_hessian_results}
\end{figure*}

Figure~\ref{fig_hessian_training_curve} shows the dynamics of the trace of the Hessian and the generalization errors during training. 
We find that they both follow a qualitatively similar trend.
Next, as shown in Figure~\ref{fig_hessian_vary_k}, we find that single-task training yields a very low generalization error in the Meta-World benchmark, indicating that a single policy can solve an individual task. However, when the task count reaches three, the generalization error sharply increases, suggesting that a single policy struggles to perform well across multiple tasks. As the number of tasks increases, the generalization error and the Hessian trace decrease, indicating the positive transfer among tasks.
In the control environments, we observe that both the generalization error and the Hessian trace decrease as $\alpha$ increases.

\subsection{Ablation Analysis}
Given that each task is solvable in the Meta-World benchmark, our goal is to determine the minimum number of groups $k$ required to maximize the success rate. We vary $k$ from $1$ to $4$, observing average success rates of $94.0\%$ when $k=3$, compared to $89.5\%$ when $k=2$ and $95.1\%$ when $k=4$. Thus, we report the final results with three groups.
For the control environments, we vary $k$ from $1$ to $5$. We find that the success rate stabilizes after $k$ reaches $3$, so we report results using three groups.

For the random projection dimension $d$, we vary it from $200$ to $1000$ and observe that values beyond $400$ yield minimal gains, so we fix $d=400$.

\section{Generalization Error Analysis via Hessians}

In this section, we present a generalization error analysis for the RL setting described above.
We present a Hessian-based technique to quantify generalization errors, applicable to any kind of policy network used in RL training.
We first describe the tools that we will use.
Given a data distribution of $\cD$, let
\[ L(f_W) = \exarg{(x,y)\sim\cD}{\ell(f_W(x), y)} \]
denote the expected loss of an input sample $x, y$.
Let $\hat L(f_W)$ denote the empirical loss given $n$ independent samples $(x_1, y_1), (x_2, y_2), \dots, (x_n, y_n)$ sampled from $\cD$:
\[ \hat L(f_W) = \frac 1 n \sum_{i=1}^n {\ell(f_W(x_i), y_i)}. \]
Let $\cW\subseteq\real^d$ denote the weight space of $W$. %
Assume that the weighted Hessian along $x$ is bounded by a fixed constant $\cH$ that does not grow with $d$ or $n$:%
\begin{align*}
    \cH \define \sup_{W\in\cW} \left( W^{\top} \exarg{(x, y)\sim\cD}{[\nabla^2 \ell(f_W(x), y)]^+} W \right), %
\end{align*}
where $[\nabla^2 \ell(f_W(x), y)]^+$ means that we truncate the negative eigen-directions of the Hessian to zero.
Then, we have the following generalization error bound, which applies uniformly to the hypothesis space $\set{f_W: W\in\cW}$.

\begin{theorem}\label{thm_hessian}
    Assume that the loss function $\ell$ is bounded between $0$ and $C$ for a fixed constant $C > 0$ on the data distribution $\cD$.
    Suppose $\ell(f_W(\cdot), \cdot)$ is twice-differentiable in $W$ and the Hessian matrix $\nabla^2 \ell(f_W(\cdot), \cdot) $ is Lipschitz-continuous within the weight space $\cW$.
    Suppose there exists a fixed bound $\cH$ on the Hessian trace over any $W \in \cW$.
    Then, for any $W$ in $\cW$, any $\epsilon >0$ small enough, and any small $\delta > 0$, with probability at least $1 - \delta$ over the randomness of $n$ samples, we have:
    \begin{align}
        L(f_W) \le (1 + \epsilon) \hat L(f_W) + (1 + \epsilon) \sqrt{\frac {C \cdot\cH}{n}} 
        + \epsilon, \label{eq_main_1}
    \end{align}
    where $\epsilon = O(n^{- 3 / 4}\log(\delta^{-1}))$ denotes the error term.
\end{theorem}

The proof of Theorem \ref{thm_hessian} is based on a PAC-Bayes analysis \cite{ju2022robust,zhangnoise}, and the details are in the Appendix.
In particular, a new aspect of this result is that we use an anisotropic perturbation in the prior and posterior distributions.
We build on a linear PAC-Bayes bound (See Theorem \ref{lemma_pac} in the Appendix), and optimize the noise distribution to obtain the $\cH$ measure.

In Figure \ref{fig_hessian_results}, we visualize the Hessian trace on the Meta-World and the control tasks.
Our results show that the Hessian trace is comparable in scale to the observed generalization errors.
In the Meta-World benchmark, the generalization error increases as the number of tasks increases from one to three, suggesting negative transfer. As more tasks are added, the error decreases, indicating positive transfer.

\section{Related Work}

\textbf{Multi-objective optimization for reinforcement learning.} Early works on multitask learning use hard parameter sharing, where a single neural network backbone is shared among all tasks, with only the final layers or heads being task-specific \cite{yu2020metaworld}.
A canonical example is the multi-head Soft Actor-Critic (SAC) architecture, which serves as a common baseline \cite{yu2020metaworld}. %

The multitask objective provides implicit regularization across multiple tasks trained together in a shared network \cite{wuunderstanding}, which can be formalized in a high-dimensional regression setting \cite{yang2025precise}.
Training the model to learn representations useful across a diverse set of tasks helps prevent overfitting.
A deeper analysis of the regularization effect behind multitask RL is an interesting question for future~work. %

\paragraph{Data attribution and model interpretability.}
Another line of work involves data modeling, which aims to trace a model's predictions back to its training data. %
One method is to retrain the model on a different subset of the dataset and estimate the Shapley value of each sample~\cite{ghorbani2019data}.
A commonly used technique is influence functions~\cite{koh2017understanding}, which use Hessians to approximate the effect of removing a sample.
Datamodels~\cite{ilyas2022datamodels} finds that a linear regression method can accurately approximate the outcome of deep nets trained with a subset of samples on computer vision tasks. 
\citet{li2023identification} introduce a linear surrogate model for multitask learning and analyze the sample complexity of linear surrogate models.
TRAK~\cite{park2023trak} further demonstrates that the approximate solution of linear regression computed with projected gradients delivers comparable results to the original linear model for ImageNet, CLIP, and BERT.
Our work contributes to this literature by applying attribution methods to design RL algorithms.

\section{Conclusion}

This paper introduces an efficient algorithm for multi-objective reinforcement learning.
The overall approach works by first training a meta-policy via multitask learning.
Then, use a gradient-estimation algorithm to adapt this meta-policy to multiple subsets of training objectives.
The key observation is that a policy can effectively adapt across different task subsets through a first-order approximation with a proper initialization.
Leveraging this, we propose a surrogate model to approximate the actual training performance of the policy, which enabled us to derive task affinities and subsequently cluster similar tasks into groups.
In extensive reinforcement learning tasks, our method consistently improves performance.
Lastly, we measure sharpness via Hessians to analyze generalization in multi-objective RL.
We hope this work inspires further studies on designing principled methods for multi-objective reinforcement learning and quantifying generalization in policy learning algorithms.
Further applying our approach to broader RL settings is another promising direction for future work.

\section*{Acknowledgments}

Thanks to the anonymous referees for their constructive feedback.
This work is partially funded by NSF award IIS-2412008.
Any views or opinions expressed herein are solely those of the authors listed, and may differ from those expressed by NSF.

\begin{refcontext}[sorting=nyt]
    \printbibliography
\end{refcontext}
\appendix

\section{Clustering Algorithms}

We provide further details on using optimization to cluster task groups based on the task affinity matrix.
For a partition into $m$ clusters, represented by binary indicator vectors $ v_1, \dots, v_m $, the objective, which corresponds to the average density of the $m$ clusters, is given by
\[ \sum_{i=1}^m \frac{v_i^\top U v_i}{v_i^\top v_i}. \]
However, directly optimizing this combinatorial objective is NP-hard. We therefore introduce a matrix variable
\[ X = \sum_{i=1}^m \frac{v_i v_i^\top}{v_i^\top v_i}. \]
The objective can then be rewritten as maximizing the inner product $ \langle U, X \rangle$. In this formulation, the number of clusters $m$ is equal to both the rank of $X$. %
Relaxing the non-convex rank constraint leads to a semi-definite program (SDP) formulation. To avoid specifying the number of clusters $m$ and to automatically determine a suitable number of groups, we penalize the trace of $X$ in the objective. This yields the following regularized optimization problem:
\begin{align}
    \max_{X \in \mathbb{R}^{n \times n}} \quad & \langle U, X \rangle - \lambda \cdot \tr[X] \label{eq_trace_reg} \\
    \text{s.t.} \quad & X e = e, \mathrm{Tr}[X] = k \notag\\
    & X \succeq 0, X \geq 0, \notag
\end{align}
where $e$ is the vector of all ones, and $\lambda > 0$ is a hyperparameter that controls the penalty on the number of clusters, effectively balancing clustering quality with the total number of groups.

Once we solve the SDP for the optimal ${X}$, we recover the discrete cluster assignments through a rounding procedure. This clustering procedure is efficient to implement in practice; solving the SDP for a typical number of tasks takes several seconds.
Thus, the computational cost is negligible compared to the cost of model training.

\section{Proof of Theorem \ref{thm_hessian}}\label{proof_pacbayes}

We provide a high-level illustration of the proof of Theorem \ref{thm_hessian}.
There are two main steps in the proof.
First, we would like to show that on the original input distribution, the gap between test and training losses remains bounded by the Hessian distance measure.
Second, we show that such a statement holds for every transformed distribution of the input distribution. Thus, we reduce the proof of equation \eqref{eq_main_1} to the first step.

\paragraph{Proof sketch.} We start with the first step. Let
\begin{align*}
    \hat L(f_W) &= \frac 1 n \sum_{i=1}^n \ell(f_W(x_i), y_i), \text{ and } \\
    L(f_W)      &= \exarg{(x, y) \sim \cD}{\ell (f_W(x), y)},
\end{align*}
denote the training (and test) losses of $f_W$, given $n$ independent samples from data distribution $\cD$.
Let $\cQ$ denote the \textit{posterior} distribution. Specifically, we consider $\cQ$ as being centered at the learned hypothesis $W$ (which could be anywhere within the hypothesis space), given by a Gaussian distribution $\cN(W, \Sigma)$, where $\Sigma$ is a $p$ by $p$ covariance matrix.
Given a sample $U\sim \cN(0, \Sigma)$, let the perturbed loss be given by 
\begin{align}
    \ell_{\cQ}(f_W(x), y) = \exarg{U}{\ell(f_{W + U}(x), y)}. \label{eq_lq}
\end{align}
Then, let $\hat L_{\cQ}(W)$ be the averaged value of $\ell_{\cQ}(f_W(\cdot), \cdot)$, taken over $n$ empirical samples from the training dataset.
Likewise, let $L_{\cQ}(W)$ be the population average of $\ell_{\cQ}(f_W(\cdot), \cdot)$, in expectation over an unseen data sample from the underlying data distribution.

Having introduced the notations, we start with the PAC-Bayes bound \cite{mcallester2013pac}, stated as follows.
\begin{theorem}\label{lemma_pac}
    Suppose the loss $\ell(f_W(x),y)$ lies in a bounded range $[0, C]$ given any $x\in\cX$ with label $y$.
    Let $\cP$ and $\cQ$ be 
    prior and posterior distributions on the weights $W$.
    For any $\beta\in (0,1)$ and $\delta\in (0,1)$, with probability at least $1-\delta$, we hhave:
    \begin{align}
        L_{\cQ}(W)\leq \frac{1}{\beta}\hat{L}_{\cQ}(W) + \frac{C \big(KL(\cQ||\cP)+\log\frac{1}{\delta}\big)}{2\beta(1-\beta)n}. \label{eq_pac}
    \end{align}
\end{theorem}

This result provides flexibility in setting $\beta$.
Our results will set $\beta$ to balance the two terms.
We will need the KL divergence between the prior $\cP$ and the posterior $\cQ$ in the PAC-Bayesian analysis. This is stated in the following classical result.

\begin{proposition}\label{prop_kl}
    Suppose $\cP = N(X, \Sigma)$ and $\cQ = N(Y, \Sigma)$ are both Gaussian distributions with mean vectors given by $X\in\real^p, Y\in\real^p$, and population covariance matrix $\Sigma \in \real^{p \times p}$.
    The KL divergence between $\cP$ and $\cQ$ is equal to
    \begin{align*}
        KL(\cQ||\cP) = \frac{1}{2}(X- Y)^\top \Sigma^{-1} (X - Y).
    \end{align*}
\end{proposition}

We are interested in the perturbed loss, $\ell_{\cQ}(f_W(x),y)$, which is the expectation of $\ell(f_{W + U}(x),y)$ over $U$.
Using a Taylor expansion, we find that
\begin{align*}
    \ell_\mathcal{Q}(f_W(x),y) - \ell(f_W(x),y)
    \leq \big\langle \Sigma, \nabla^2 \ell(f_W(x), y)\big\rangle + \epsilon, %
\end{align*}%
where $\Sigma$ is the population covariance matrix of the perturbation, $\nabla^2$ is the Hessian matrix with respect to the weights of $f_W$, and $\epsilon = C_1 (\bigtr{\Sigma})^{3 / 2}$.
See Lemma \ref{lemma_perturb} for the complete statement of this result.

Based on the above expansion, next, we apply the PAC-Bayes bound from equation \eqref{eq_pac} to an $L$-layer transformer neural network $f_W$ parameterized by $W$.
We note that the KL divergence between the prior and posterior distributions, which are both Gaussian, is equal to $\biginner{\Sigma^{-1}}{v v^{\top}}$, where $v$ is the difference between the initialized and the trained weights.

Next, we combine the above Hessian-based bound and KL divergence in the PAC-Bayes bound. %
Let ${{\nabla}^2_{+}}$ denote the truncated Hessian matrix where we set the negative eigenvalues of $\nabla^2$ to zero.
We have that
{\begin{align}
    \inner{\Sigma}{{\nabla^2_{W}[\ell(f_W(x), y)]}} \le \inner{\Sigma}{{{{\nabla}^2_{+}}[\ell(f_W(x), y)]}}. \label{eq_cov}
\end{align}}%
Substituting into \eqref{eq_pac}, and minimizing over $\beta$ and $\Sigma$, we will derive an upper bound on the generalization error (between $L(f_W)$ and $\hat L(f_W)$) equal to
\begin{align}
    \alpha \define \sup_{W\in\cW} \sup_{(x, y)\sim \cD} \frac{ \sqrt{W^{\top}\big[{\tilde{\nabla}_+^2}{\ell}(f_W(x), y) \big] W} } {\sqrt{n}}, \label{eq_hess}
\end{align}%
where $\mathcal{W}$ is the support of the distribution from which the data is drawn.

We will use a Taylor expansion to bound the perturbed loss, as follows.
\begin{claim}\label{lemma_taylor}
    Let $f_W$ be twice-differentiable, parameterized by weight vector $W \in\real^p$.
    Let $U\in\real^p$ be another vector with dimension $p$.
    For any $W$ and $U$, the following identity holds
    \begin{align*}%
        \ell(f_{W + U}(x), y) = \ell(f_W(x), y) +  U^{\top} \nabla \ell(f_W(x), y) 
        + \frac 1 2 {U}^{\top} [\nabla^2\ell(f_W(x),y)] {U} + R_2(\ell(f_W(x),y)),
    \end{align*}
    where $R_2(\ell(f_W(x),y)))$ is a second-order error term in a Taylor expansion of $\ell\circ f_W$ around $W$.
\end{claim}

\begin{proof}
    The proof follows from the fact that $\ell\circ f_W$ is twice-differentiable.
    By the mean value theorem,
    there must exist $\eta$ between $W$ and $U+W$ such that
    \begin{align*}
        R_2(\ell(f_W(x),y))
        = {U}^{\top} \Big(\nabla^2[\ell(f_{\eta}(x), y)] - \nabla^2[\ell(f_W(x), y)]\Big){U}.
    \end{align*}
    This completes the proof of the claim.
\end{proof}

Based on the above result, we provide a Taylor's expansion for $\ell_{\cQ}$ minus $\ell$.

\begin{lemma}\label{lemma_perturb}
    In the setting of Theorem \ref{thm_hessian}, suppose each parameter is perturbed by an independent random variable drawn from $N(0,\Sigma)$.
    Let ${\ell}_{\cQ}(f_W(x),y)$ be the loss averaged over the noise.
    There is a value $C_1$ that does not depend on $n$ and $1/\delta$ such that
    \begin{align}
        \ell_\mathcal{Q}(f_W(x),y) - \ell(f_W(x),y)  
        \leq \frac 1 2 \big\langle \Sigma,\nabla^2[\ell(f_W(x),y)]\big\rangle + C_1(\bigtr{\Sigma})^{\frac 3 2}.\label{eq_taylor_app}
    \end{align}
\end{lemma}
 
\begin{proof}%
We take the expectations over $U$ of both sides of the equation in Claim \ref{lemma_taylor}. The result becomes
\begin{align*}
    \exarg{U}{\ell(f_{W+U}(x), y)}
    = \exarg{U}{\ell(f_W(x), y) + U^{\top} \nabla \ell(f_W(x), y) + \frac 1 2 {U}^{\top}\nabla^2[\ell(f_W(x), y)]{U} + R_2(\ell(f_W(x),y))}.
\end{align*}
Then, we use the perturbation distribution to calculate
\begin{align*}
    \ell_\cQ(f_W(x),y) 
     = \exarg{U}{\ell(f_W(x),y)} + \exarg{U}{U^{\top} \nabla\ell(f_W(x), y)} + \frac 1 2 \exarg{U}{{U}^{\top}\nabla^2[\ell(f_W(x), y)]{U}} 
    + \exarg{U}{R_2(\ell(f_W(x), y))}.
\end{align*}
Since $\mathbb{E}[U] = 0$, the first-order term vanishes.
The second-order term becomes
\begin{align}
    \exarg{U}{{U}^{\top} [\nabla^2\ell(f_W(x), y)] {U}} = \inner{\Sigma}{{\nabla^2[\ell(f_W(x),y)]}}. \label{eq_taylor_2}
\end{align}
The expectation of the error term $R_2(\ell(f_W(x),y))$ be
\begin{align*}
    \exarg{U}{R_2(\ell(f_W(x), y))} &= \exarg{U}{U^{\top}\bigbrace{{\nabla^2[\ell(f_{\eta}(x), y)] - \nabla^2[\ell(f_W(x), y)]}} U} \\
    &\le \exarg{U}{\bignorms{U}^2 \cdot \bignormFro{\nabla^2[\ell(f_{\eta}(x), y)] - \nabla^2[\ell(f_W(x), y)]}} \\
    &\lesssim \exarg{U}{\bignorms{U}^2 \cdot C_1 {\bignorms{U}}} \\
    &\lesssim C_1 \left(\ex{U^{\top} U}\right)^{\frac 3 2} = C_1 (\bigtr{\Sigma})^{\frac 3 2}.
\end{align*}
To obtain the first inequality in the last line, we have used that $\eta$ lies on the line between $W$ and $W+U$, so that $\|\eta\| \le \|U\|$, and that the Hessian is Lipschitz.
Thus, the proof is complete.
\end{proof}

The last piece we will need is the uniform convergence of the Hessian, which uses the fact that the Hessian is Lipschitz continuous.

\begin{lemma}\label{lemma_union_bound}
    In the setting of Theorem \ref{thm_hessian}, there exist values $C_2, C_3$ that do not grow with $n$ and $1/\delta$, such that for any $\delta > 0$, with probability at least $1- \delta$ over the randomness of the $n$ training examples, we have
    \begin{align}
        \bignormFro{\frac 1 n \sum_{i=1}^n\nabla^2[\ell(f_W(x_i), y_i)] - \exarg{(x,y)\sim\cD}{\nabla^2[\ell(f_W(x), y)]}} \le \frac{C_2\sqrt{\log (C_3 n/\delta)}}{\sqrt n}.
    \end{align}
\end{lemma}
The proof will be deferred to Section \ref{proof_uniform}.
With these results ready, we will provide proof of the Hessian-based generalization bound.

Now, we present the full proof of Theorem \ref{thm_hessian}.

\begin{proof}[Proof of Theorem \ref{thm_hessian}]
    First, we separate the gap between $L({W})$ and $\frac{1}{\beta}\hat{L}({W})$ into three parts:
    \begin{align*}
        L({W}) - \frac{1}{\beta} \hat{L}({W})
        = L({W}) - L_\cQ({W}) + L_\cQ({W}) - \frac{1}{\beta}\hat{L}_\cQ({W}) + \frac{1}{\beta}\hat{L}_\cQ({W}) - \frac{1}{\beta}\hat{L}({W}).
    \end{align*}
    By Lemma \ref{lemma_perturb}, we can bound the difference between $L(W)$ and $L_{\cQ}(W)$ by the Hessian trace plus an error:
    \begin{align*}
        L({W}) - \frac{1}{\beta} \hat{L}({W})
        \leq& -\exarg{(x,y)\sim\cD}{\frac{1} 2 \inner{\Sigma}{\nabla^2[\ell(f_{W}(x), y)]}} + C_1(\bigtr{\Sigma})^{\frac 3 2} + \Big(L_\cQ({W}) - \frac{1}{\beta}\hat{L}_\cQ({W})\Big)\\
        &+ \frac{1}{\beta}\Big(\frac{1}{n}\sum_{i=1}^n \frac{1} 2 \inner{\Sigma}{\bigtr{\nabla^2[\ell(f_{W}(x_i), y_i)]}} + C_1(\bigtr{\Sigma})^{\frac 3 2}\Big).
    \end{align*}
    After rearranging the terms, we find
    \begin{align}
        L({W}) - \frac{1}{\beta} \hat{L}({W}) \leq& \underbrace{-\exarg{(x,y)\sim\cD}{\frac{1}2 \inner{\Sigma}{{ \nabla^2[\ell(f_{W}(x), y)]}}}
         + \frac{1}{n\beta} \sum_{i=1}^n\frac{1}2\inner{\Sigma}{{\nabla^2[\ell(f_{W}(x_i), y_i)] }}}_{E_1} \nonumber \\
        & + \frac{1 + \beta}{\beta} C_1 {\bigtr{\Sigma}}^{\frac 3 2}
        + \underbrace{L_\cQ({W}) - \frac{1}{\beta}\hat{L}_\cQ({W})}_{E_2}. \label{eq_combine_3_hess}
    \end{align}
    We analyze $E_1$ by separating it into two parts
    \begin{align}
        E_1 =& \frac{1}{\beta}\left(\frac{1}{n}\sum_{i=1}^n \frac{1}2\inner{\Sigma}{ \nabla^2[\ell(f_{\hat W}(x_i), y_i)] } - \exarg{(x,y)\sim\cD}{\frac{1}2 \inner{\Sigma}{\nabla^2[\ell(f_{W}(x), y)]}}\right) \label{eq_unif_conv} \\
         &+ \frac {1 - \beta} {2\beta} \exarg{(x,y)\sim\cD}{\inner{\Sigma}{\nabla^2\ell(f_{W}(x), y)}}.\label{eq_chernoff_trace_hess}
    \end{align}
    We can use the uniform convergence result of Lemma \ref{lemma_union_bound} to bound the term in \eqref{eq_unif_conv}, leading to:
    \begin{align}
        & \frac{1}{2\beta}\left(\frac{1}{n}\sum_{i=1}^n \inner{\Sigma}{\nabla^2\ell(f_{{W}}(x_i), y_i)} - \exarg{(x,y)\sim\cD}{\inner{\Sigma}{\nabla^2\ell(f_{{W}}(x),y))}}\right) \nonumber \\
        \le& \frac{\sigma^2}{2\beta} \cdot \sqrt{p} \cdot \bignormFro{\frac 1 n \sum_{i = 1}^n{ \nabla^2[\ell(f_{{W}}(x_i), y_i)] } - \exarg{(x,y)\sim\cD}{{ \nabla^2[\ell(f_{{W}}(x), y)] }}} \tag{by Cauchy-Schwarz} \\
        \le& \frac{\sigma^2\sqrt p \cdot C_2\sqrt{\log(C_3 n/\delta)}}{2\beta \sqrt n}.\label{eq_con_err_hess}
    \end{align}
    In particular, the second step also uses the fact that the Hessian is a symmetric $p$ by $p$ matrix.
    As for equation \eqref{eq_chernoff_trace_hess}, we have that
    \begin{align*}
        \frac{1 - \beta}{2\beta} \exarg{(x, y)\sim\cD}{\inner{\Sigma}{\nabla^2 \ell(f_W(x), y)}}
        \le \frac{1 - \beta}{2\beta} \inner{\Sigma}{\exarg{(x,y)\sim\cD}{\nabla_+^2 \ell(f_W(x), y)}},
    \end{align*}
    Combined with equation \eqref{eq_con_err_hess}, we have shown that
    \begin{align}
        E_1 \le \frac{\sigma^2 \sqrt p \cdot C_2 \sqrt{\log(C_3 n /\delta)}}{2\beta\sqrt n} + \frac{1-\beta}{2\beta} \inner{\Sigma}{\exarg{(x, y)\sim\cD}{\nabla_+^2\ell(f_W(x), y)}}. \label{eq_E1_err}
    \end{align}
    Next, for $E_2$, we use the PAC-Bayes bound of Theorem \ref{lemma_pac}.
    In particular, we set the prior distribution $\cP$ as the distribution of $U$ and the posterior distribution $\cQ$ as the distribution of $W + U$.
    Thus,
    \begin{align}
        E_2 \le \frac{C \big(KL(\cQ || \cP) + \log(\delta^{-1}) \big)}{2\beta(1 - \beta) n} \le \frac{C \Big(\frac 1 2 { {W}^{\top} \Sigma^{-1} W} + \log(\delta^{-1}) \Big)} {2\beta (1 - \beta) n}. \label{eq_combine_1_hess}
    \end{align}
    Combining equations \eqref{eq_combine_3_hess}, \eqref{eq_E1_err}, \eqref{eq_combine_1_hess}, we claim that with probability at least $1 - 2\delta$, we must have:
    \begin{align}
        L({W}) - \frac{1}{\beta}\hat{L}({W}) \leq&  
        \frac{\sigma^2\sqrt p \cdot C_2 \sqrt{\log(C_3 n/\delta)}}{2 \beta\sqrt{n}}
        + \frac{1 + \beta}{\beta} C_1{\bigtr{\Sigma}}^{3/2} + \frac{C \log\frac{1}{\delta}}{2\beta(1-\beta)n} \label{eq_gap1} \\
        & + \frac{C W^{\top} \Sigma^{-1} W} {4\beta(1 - \beta) n} + \frac{ 1- \beta}{2\beta} \inner{\Sigma}{\bH_W}, \label{eq_balance}
    \end{align}
    where \[ \bH_W = \exarg{(x,y)\sim\cD}{ [\nabla^2 \ell(f_W(x), y)]^+ }. \]
    We next choose $\Sigma$ and $\beta \in (0,1)$ to minimize the above bound.
    In particular, we set
    \begin{align}
        \Sigma = \sqrt{\frac{C}{2(1-\beta)^2 n \bignorm{W}^2}} {\bH_W}^{-\frac 1 2} WW^{\top} \label{eq_equal_hess}
    \end{align}
    so that the two terms in equation \eqref{eq_balance} are equal to each other:
    \begin{align*}
        \frac{1-\beta}{\beta} \cdot \sqrt{\frac C {2(1-\beta)^2 n\bignorm{W}^2}} \inner{{\bH_W}^{\frac 1 2} W} {W}
        &\le \frac{1}{\beta}\sqrt{\frac C {2n \bignorm{W}^2}} \bignorm{{\bH_W}^{\frac 1 2} W} \bignorm{W} \\
        &= \frac 1 {\beta} \sqrt{\frac C {2n}} \bignorm{W^{\top} \bH_W W} \\
        &\le \frac {1} {\beta} \sqrt{\frac {C \cH} {2n}}.
    \end{align*}
    
    By plugging in $\sigma$ to equation \eqref{eq_gap1} and re-arranging terms, the gap between $L(W)$ and ${\beta}^{-1} {\hat{L}(W)}$ can be bounded as:
    \begin{align*}
         L(W) - \frac{1}{\beta} \hat{L}(W) 
        \le \frac{1}{\beta} \sqrt{\frac{C\cH }{n}} 
        + \frac{C_2\sqrt{2 p \log (C_3 n/\delta)}}{2\beta \sqrt n} \sigma^2 + \frac{1 + \beta}{\beta} C_1 {\bigtr{\Sigma}}^{3/2} + \frac{C}{2\beta(1-\beta)n}\log \frac{1}{\delta}.
    \end{align*}
    Let $\beta = 1 / (1 + \epsilon)$ and so that $\epsilon = (1 - \beta)/\beta$.
    We find that
    \begin{align*}
        L(W) &\le  (1 + \epsilon) \hat{L}(W) + (1 + \epsilon) \sqrt{\frac{C \cH}{n}} + \xi, \text{ where } \\
        &\quad\xi = \frac{C_2\sqrt{2 p \log (C_3 n/\delta)}}{2\beta \sqrt n}  \sigma^2 + \Bigbrace{1 + \frac{1}{\beta}} C_1 {\bigtr{\Sigma}^{3/2}} + \frac{C}{2\beta(1-\beta)n}\log \frac{1}{\delta}.
    \end{align*}
    Notice that $\xi$ is of order 
    \[ O(n^{-\frac 3 4} + n^{-\frac 3 4} + \log(\delta^{-1}) n^{-1}) = O(\log(\delta^{-1})  n^{-\frac 3 4}). \]
    Therefore, we have shown that with probability at least $1 - \delta$,
    \begin{align}
        L(f_W) \le (1 + \epsilon) \hat L(f_W) + (1 + \epsilon) \sqrt{\frac{C \cH}{ n}} + O(n^{-3/4}).
    \end{align}
    The proof is now complete.
\end{proof}

\subsubsection{Proof of Lemma \ref{lemma_union_bound}}\label{proof_uniform}

In this section, we provide the proof of Lemma \ref{lemma_union_bound}, which shows the uniform convergence of the loss Hessian.

\begin{proof}[Proof of Lemma \ref{lemma_union_bound}]
    Let $C$, $\epsilon > 0$, and let \[ S = \{W\in\real^{p}: \bignorms{W}\leq C\}. \]
    There exists an $\epsilon$-cover of $S$ with respect to the $\ell_2$-norm with at most $\max\Big(\big(\frac{3C}{\epsilon}\big)^p,1\Big)$ elements; see, e.g., Example 5.8 \citep{wainwright2019high}.
    Let $T\subseteq S$ denote this cover.
    Recall that the Hessian $\nabla^2[\ell(f_W(x), y)]$ is  $C_1$-Lipschitz for all $(W+U) \in S, W\in S$. Then we have
    \begin{align*}
        \bignormFro{\nabla^2[\ell(f_{W+U}(x),y)] - \nabla^2[\ell(f_W(x),y)]}\leq C_1 \bignorms{U}.
    \end{align*}
    For $\delta,\epsilon > 0$, 
    define the event 
    \begin{align*}
        E=\Big\{\forall~ W\in T, \bignormFro{\frac{1}{n}\sum_{i=1}^n \nabla^2[\ell(f_{W}(x_i), y_i)] - \exarg{(x,y)\sim\cD}{ \nabla^2[\ell(f_W(x), y)]}}\leq\delta\Big\}.
    \end{align*}
    By the matrix Bernstein inequality \cite{vershynin2018high}, we have
    \begin{align} \Pr[E] \geq 1 - 4\cdot |\cN|\cdot p\cdot \exp\left(-\frac {n\delta^2} {2\alpha^2}\right). \end{align}
    Next, for any $W\in S$, we can pick some $W + U\in T$ such that $\bignorms{U}\leq \epsilon$. We have
    \begin{align*}
        &\bignormFro{\exarg{(x,y)\sim\cD}{\nabla^2[\ell(f_{W+U}(x), y)]} - \exarg{(x,y)\sim\cD}{\nabla^2[\ell(f_W(x), y)]}}\leq C_1\bignorms{U}\leq C_1\epsilon\\
        &\bignormFro{\frac{1}{n}\sum_{j=1}^n\nabla^2[\ell(f_{W+U}(x_j), y_j)] -  \frac{1}{n}\sum_{j=1}^n\nabla^2[\ell(f_W(x_j),y_j)]}\leq C_1\bignorms{U}\leq C_1\epsilon.
    \end{align*}
    Therefore, for any $W\in S$, we obtain:
    \begin{align*}
        \bignormFro{\frac{1}{n}\sum_{j=1}^n\nabla^2[\ell(f_{W}(x_j), y_j)] - \exarg{(x,y)\sim\cD}{ \nabla^2[\ell(f_W(x), y)]}}\leq 2C_1\epsilon + \delta.
    \end{align*}
    Set $\epsilon = \delta/(2C_1)$ so that on the event $E$,
    \begin{align*}
        \bignormFro{\frac{1}{n}\sum_{j=1}^n\nabla^2[\ell(f_{W}(x_j), y_j)] - \exarg{(x,y)\sim\cD}{ \nabla^2[\ell(f_W(x), y)]}}\leq 2\delta.
    \end{align*}
    The event $E$ happens with a probability of at least:
    \begin{align*}
        1 - 4|T|p \cdot\exp\left(-\frac{ n\delta^2} {2\alpha^2}\right) = 1 - 4p\cdot \exp\left(\log |T| - \frac{n\delta^2}{2\alpha^2}\right).
    \end{align*}
    Now, we have $\log |T| \leq p\log(3B/\epsilon) = p\log (6C C_1 /\delta)$.
    If we set
    \begin{align} \delta = \sqrt{\frac{4p\alpha^2\log(3\tau C C_1 n/\alpha)}{n}},
    \end{align}
    so that 
    $\log(3\tau C C_1 n/\alpha) \geq 1$ (because $n\geq \frac{e\alpha}{3C_1}$ and $\tau\geq 1$), then we find
    \begin{align*}
         p\log(6C C_1/\delta) - n\delta^2/(2\alpha^2)
        =&p\log\fullbrace{\frac{6C C_1\sqrt{n}}{\sqrt{4p\alpha^2\log(3\tau C C_1 n/\alpha)}}} - 2p\log\fullbrace{3\tau C C_1 n/\alpha} \\
        =&p\log\fullbrace{\frac{3C C_1\sqrt{n}}{\alpha\sqrt{p\log(3\tau C C_1 n/\alpha)}}} - 2p\log\fullbrace{3\tau C C_1 n/\alpha} \\
        \leq& p\log\fullbrace{3\tau C C_1 n/\alpha} - 2p\log\fullbrace{3\tau C C_1 n/\alpha} \tag{$\tau\ge 1,\log(3\tau C C_1 n/\alpha)\geq 1$}\\
        =& -p\log\fullbrace{3\tau C C_1 n/\alpha} \leq -p\log(e\tau) \tag{$3C C_1 n/ \alpha\geq e$}.
    \end{align*}
    Therefore, with a probability at least
    \begin{align} 1 - 4|\cN|p\cdot\exp(-n\delta^2/(2\alpha^2))\geq 1 - 4p (e\tau)^{-p}, \end{align}
   we have
    \begin{align*}
        \bignormFro{\frac{1}{n}\sum_{j=1}^n \nabla^2[\ell(f_{W}(x_j),y_j)] - \exarg{(x,y)\sim\cD}{ \nabla^2[\ell(f_W(x), y)]}}\leq \sqrt{\frac{16p\alpha^2\log(3\tau C C_1 n/\alpha)}{n}}.
    \end{align*}
    Denote $\delta' = 4p(e\tau)^{-p}$, $C_2 = 4\alpha\sqrt{p}$, and $C_3 = 12 p C C_1 /(e\alpha)$. 
    With probability at least $1 - \delta'$, we have
    \begin{align*}
        \bignormFro{\frac{1}{n}\sum_{i=1}^n \nabla^2[\ell(f_{W}(x_i), y_i)] - \exarg{(x,y)\sim\cD}{ \nabla^2[\ell(f_W(x), y)]}}\leq C_2\sqrt{\frac{\log(C_3 n/\delta')}{n}}.
    \end{align*}
    This completes the proof of Lemma \ref{lemma_union_bound}.
\end{proof}

\paragraph{Discussions.} 
The minimax optimal sample complexity to find an $\epsilon$-optimal $Q$-function in tabular $\gamma$-discounted MDPs is known to be $\tilde{O}\left({SA}{(1-\gamma)^{-3} \epsilon^{-2}}\right)$~\cite{chi2025statistical}. This rate is achieved by model-based algorithms. Standard synchronous $Q$-learning, however, is sub-optimal, requiring $\tilde{O}\left({SA}{(1-\gamma)^{-4} \epsilon^{-2}}\right)$ samples. This gap in horizon dependency was later closed by variance reduction techniques. For instance, a variance-reduced $Q$-learning algorithm~\cite{wainwright2019variance} achieves the minimax optimal rate, matching the lower bound up to logarithmic factors. See a recent monograph for extensive references \cite{foster2023foundations}.
It is an interesting open question to study the sample complexity of $Q$-learning for multi-objective reinforcement learning.
Another interesting question is to explore gradient estimation~\cite{li2025efficient, zhang2025linear} for designing ensemble methods in meta-RL.
These are left for future work.

\section{Omitted Experimentals}

\textbf{Environments.} 
We summarize the RL environments: MT10, CartPole, Highway, LunarLander in Table~\ref{table_envrionment_summarization}.
We adopt the Gymnasium library~\cite {towers2024gymnasium} for environment implementation.

\begin{table}[t]
\caption{A summary of RL environments tested in our experiments.}\label{table_envrionment_summarization}
{
\begin{tabular}{p{2cm}p{3.0cm}p{3.0cm}p{3.0cm}p{3.0cm}}
    \toprule
    \textbf{Env.} & \textbf{Goal} & \textbf{State} & \textbf{Action} & \textbf{Reward} \\ \hline
    MT10 & Execute robotic manipulation tasks & $\mathbb{R}^{39}$, representing the state of the robot and the target objective. & 4 continuous actions: controlling end-effector movement and gripper actuation. & A sparse reward for moving the objective to its goal position. \\ \hline
    CartPole & Balance a pole attached to a cart  & $\mathbb{R}^4$, representing the cart's position, velocity, the pole's angle, and angular velocity.  & 2 discrete actions: applying a push to the cart (left or right). & A reward of +1 for each timestep the pole remains upright. \\ \hline
    Highway  & Avoid collisions while driving on a highway & $\mathbb{R}^{5\times 5}$, representing the kinematic states of the ego vehicle and four nearby vehicles. & 5 discrete actions: change lane left/right, maintain lane, accelerate, decelerate. & A dense reward for maintaining high speed, with a large penalty for collisions. \\ \hline
    LunarLander & Land a spacecraft safely on the ground  & $\mathbb{R}^8$, representing the lander's coordinates, velocities, angle, and leg contact status.  & 4 discrete actions: control the main engine and side thrusters. & A shaped reward for approaching the landing pad, plus a large bonus for a soft landing.    \\
    \bottomrule
\end{tabular}}
\end{table}

\paragraph{RL algorithms.}
In the Meta-World experiments, we employ Soft Actor-Critic (SAC)~\cite{haarnoja2018soft} as the policy gradient algorithm.
SAC is an off-policy, actor-critic deep reinforcement learning algorithm based on the maximum entropy reinforcement learning framework. The core idea is to encourage exploration by maximizing a trade-off between the expected return and the policy's entropy.

The SAC algorithm learns three functions: A policy $\pi_\theta$ and two critic functions $Q_{\psi_1}$ and $Q_{\psi_2}$. 
We train the critic networks by minimizing the mean squared Bellman error:
\begin{align}
    L(\psi_i) = \hat{\mathbb{E}}_{(s_t, a_t, r_t, s_{t+1}, d_t) \sim \mathcal{D}} \left[ \frac{1}{2} \left( Q_{\psi_i}(s_t, a_t) - y(r_t, s_{t+1}, d_t) \right)^2 \right].
\end{align}
The target value $y$ is shared between both critics and is computed as:
\begin{align}
    y(r_t, s_{t+1}, d_t) = r_t + \gamma (1 - d_t) \left( \min_{j=1,2} Q_{\psi_{j, \text{target}}}(s_{t+1}, a_{t+1}) - \alpha \log \pi_\theta(a_{t+1} \mid s_{t+1}) \right), 
\end{align}
with $a_{t+1} \sim \pi_\theta(\cdot \mid s_{t+1})$.
The target network parameters $\psi_{j, \text{target}}$ is updated by \[ \psi_{j, \text{target}}\leftarrow\tau\psi_j+(1-\tau)\psi_{j, \text{target}}. \]
Next, the actor parameters $\theta$ are trained to maximize the policy's expected return and entropy:
\begin{align}
    L_{\text{SAC}}(\theta) = \mathbb{E}_{\hat{s}_t \sim \mathcal{D}} \left[ 
    \mathbb{E}_{a_t \sim \pi_\theta(\cdot \mid s_t)} \left[
        \alpha \log \pi_\theta(a_t \mid s_t) - \min_{j=1,2} Q_{\psi_j}(s_t, a_t)
    \right]
\right]
\end{align}

In the control environments, we employ Proximal Policy Optimization (PPO)~\cite{schulman2017proximal} as our policy gradient algorithm.
PPO follows the design of the actor-critic algorithm~\cite {konda1999actor}, where the core optimization objective is the clipped surrogate policy loss.
This is combined with a value function loss and an entropy as the optimization objective.

PPO involves tackling the following objective:
\begin{align}\label{eq_ppo_loss}
    L_{\text{PPO}}(\theta) = \mathbb{E}_{t} \left[ \min \left(r_t(\theta)\hat{A}_t, \text{clip}(r_t(\theta), 1 - \epsilon, 1 + \epsilon) \hat{A}_t \right)
    \right] - c_1 \mathbb{E}_{t} \left[ (V_\theta(s_t) - \hat{V}_t)^2 \right]
    + c_2 \mathbb{E}_{t} \left[ \mathcal{H}[\pi_\theta](s_t) \right],
\end{align}
where \[ r_t(\theta_t)=\frac{\pi_{\theta_t}(a_t|s_t)}{\pi_{\theta_{t-1}}(a_t|s_t)} \] is the probability ratio between the new policy and the previous policy, $\hat{A}_t$ is the advantage estimation computed from the value function $V_\theta(s_t)$, and $\mathcal{H}[\pi_\theta](s_t)$ is the entropy of the policy.

In the meta-RL setting, we use the MAML algorithm~\cite{finn2017model}, which consists of a two-step optimization process.
In the inner loop adaptation, MAML computes an adapted policy for each source task $T_i$ from the trajectory $\mathcal{D}_\theta^i$ sampled by the interaction with the meta policy $\theta$ and the corresponding environment: 
\begin{align}
    \theta_i'=\theta-\alpha\nabla_{\theta}L_{T_i}(\pi_{\theta}).
\end{align}

In the outer loop meta-update, MAML then aggregates the data from the trajectory sampled by the interaction with the task-specific policy $\phi_1^i$ and the corresponding environment:
\begin{align}
    \theta \leftarrow \theta + \beta\cdot \nabla_{\theta}\sum_{T_i\in\mathcal{T}} L_{T_i}(\pi_{\theta}).
\end{align}

\paragraph{Baseline descriptions.} 
PaCo (Parameter Compositional MTRL) introduces a highly flexible framework where task-specific parameters are explicitly composed rather than separated \cite{sun2022paco}. In PaCo, the parameter vector for a given task %
is constructed as a linear combination of a shared, task-agnostic basis of parameter vectors and a learned, task-specific compositional vector. %
Other notable architectures, such as Soft Modularization (SM) and those based on Mixture-of-Experts (MoE), explore ways to dynamically route information and learn task-specific components in a unified network \cite{yang2020multi, sun2022paco}.

Simultaneous optimization of a shared set of parameters under competing demands of various tasks frequently leads to destructive interference, where joint training results in worse performance on a given task than training it alone \cite{yu2020pcgrad, teh2017distral}.
The gradients from different tasks can point in conflicting directions, which destabilizes learning \cite{yu2020pcgrad, liu2021cagrad}.

\paragraph{Meta-reinforcement learning.}
Meta reinforcement learning may be viewed as a specific scenario of meta learning.
MAML~\cite{finn2017model} is a foundational work that introduces a bi-level optimization structure. In the inner loop, the policy is adapted to specific source tasks. In the outer loop, the initial meta-policies are updated by differentiating with respect to inner-loop adaptation, optimizing for the expected performance of the adapted policy across a distribution of source tasks.
Building on this, ProMP~\cite{rothfusspromp} introduces a low variance curvature estimator to stabilize training by assigning weight to each sample in trajectories.
Tr-MAML~\cite{collins2020task} focuses on the most challenging task and introduces a min-max framework to minimize the maximum loss in the set.

\paragraph{Implementation details.}
In the Meta-World benchmark, we use training steps of $10^7$, a buffer size of $10^6$, a batch size of $1280$, and a learning rate of $3\times10^{-4}$ with the Adam optimizer.

In the CartPole, Highway, and LunarLander environments, we use the following settings: $100$ meta-training steps, $2048$ training steps for each inner loop policy update, and a batch size of $64$.

\begin{figure}[t!]
    \centering
    \begin{subfigure}[b]{0.45\textwidth}
        \centering
        \includegraphics[width=0.99\textwidth]{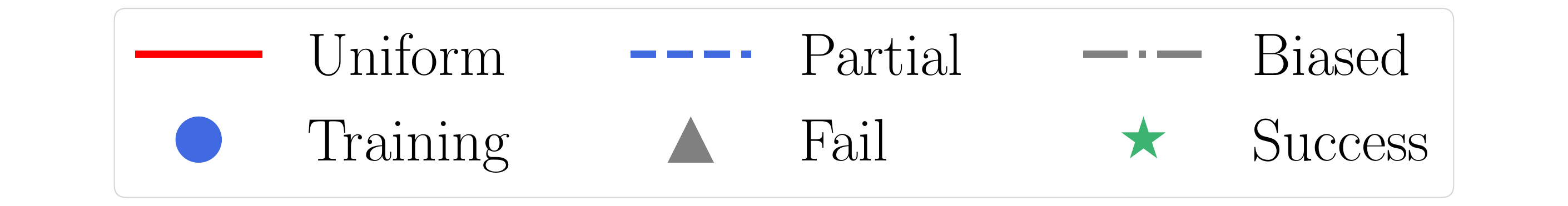}
    \end{subfigure}
    
    \begin{subfigure}[b]{0.23\textwidth}
        \centering
        \includegraphics[width=0.99\textwidth]{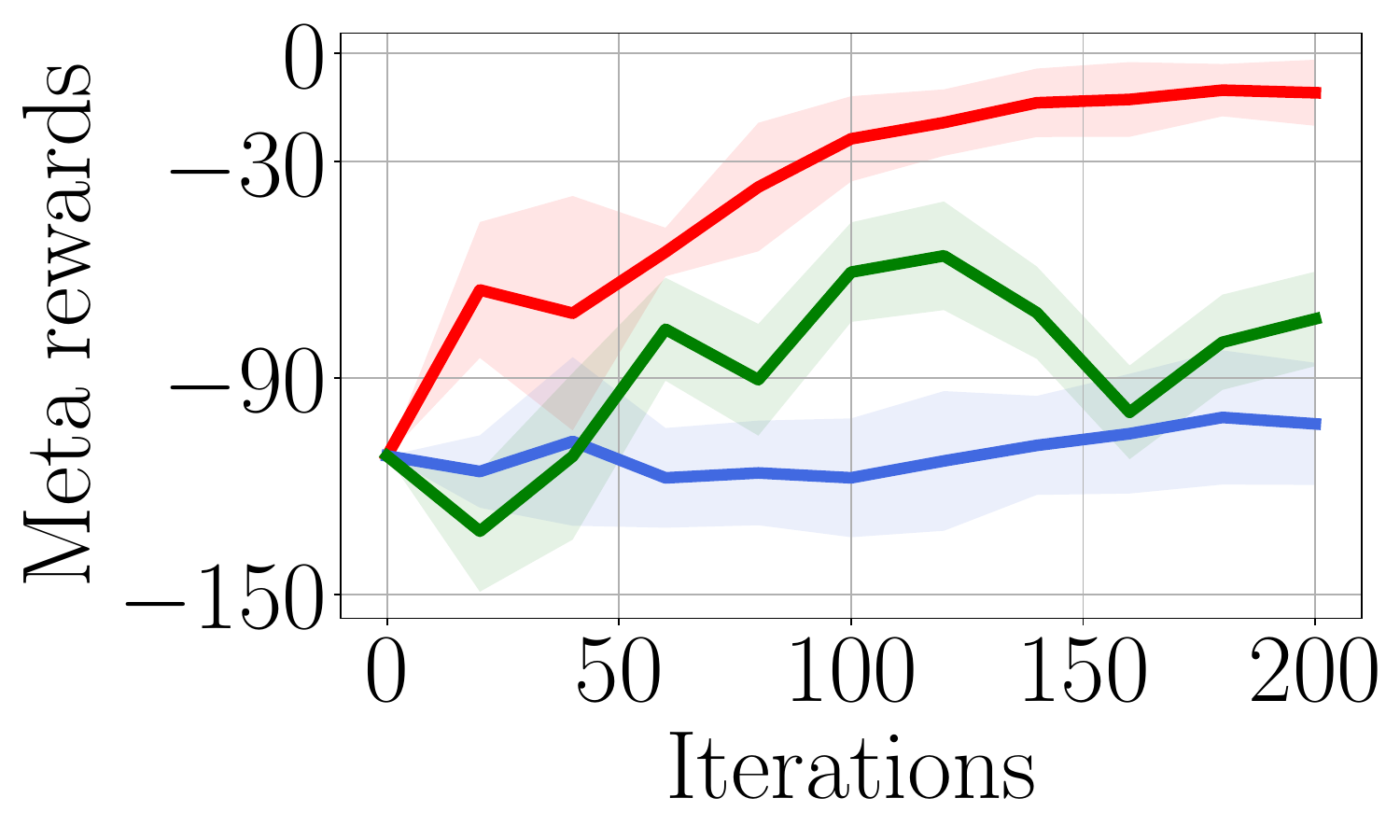}
        \caption{Meta training curve}\label{fig_navigation_curve}
    \end{subfigure}\hfill
    \begin{subfigure}[b]{0.23\textwidth}
        \centering
        \includegraphics[width=0.99\textwidth]{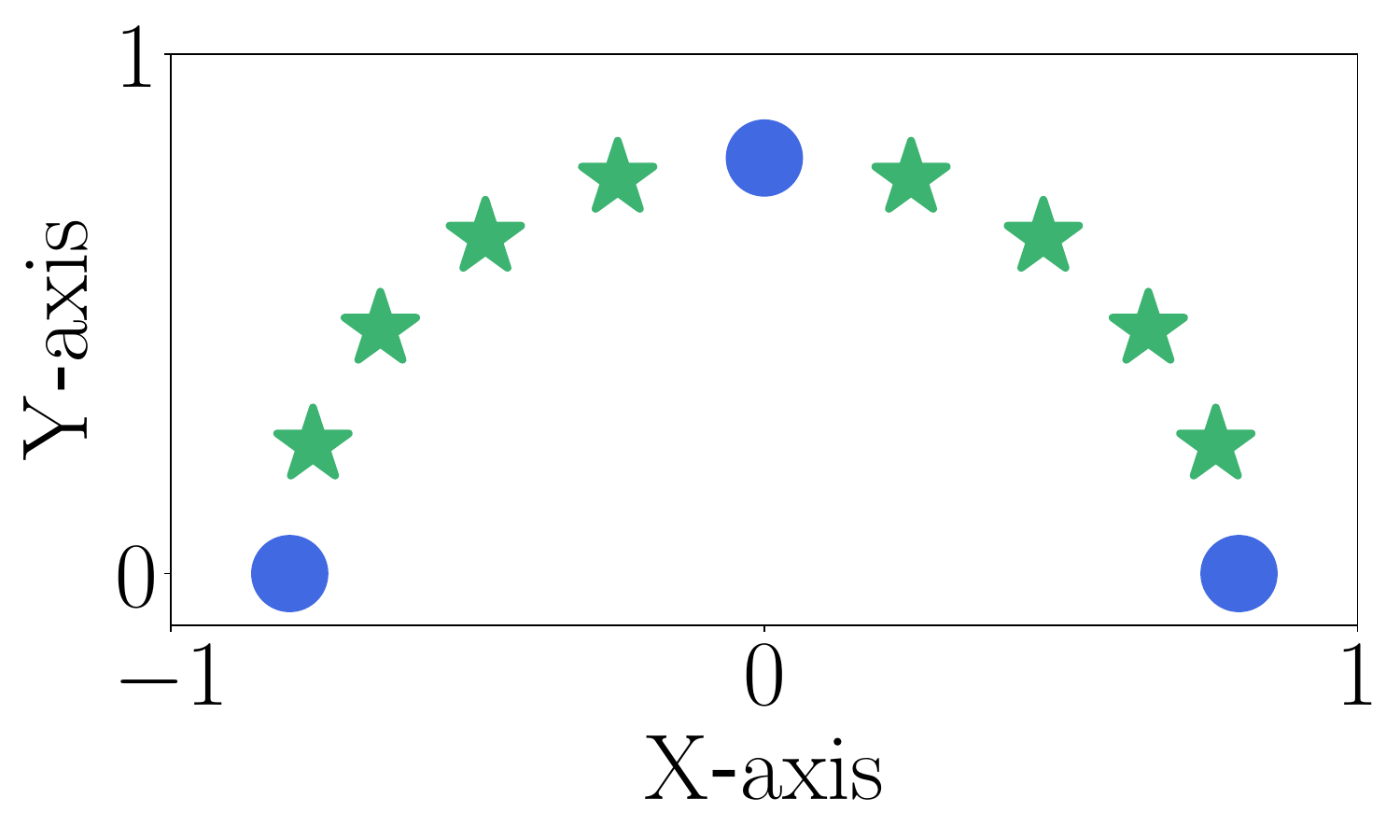}
        \caption{Uniform sampling}\label{fig_navigation_uniform}
    \end{subfigure}\hfill 
    \begin{subfigure}[b]{0.23\textwidth}
        \centering
        \includegraphics[width=0.99\textwidth]{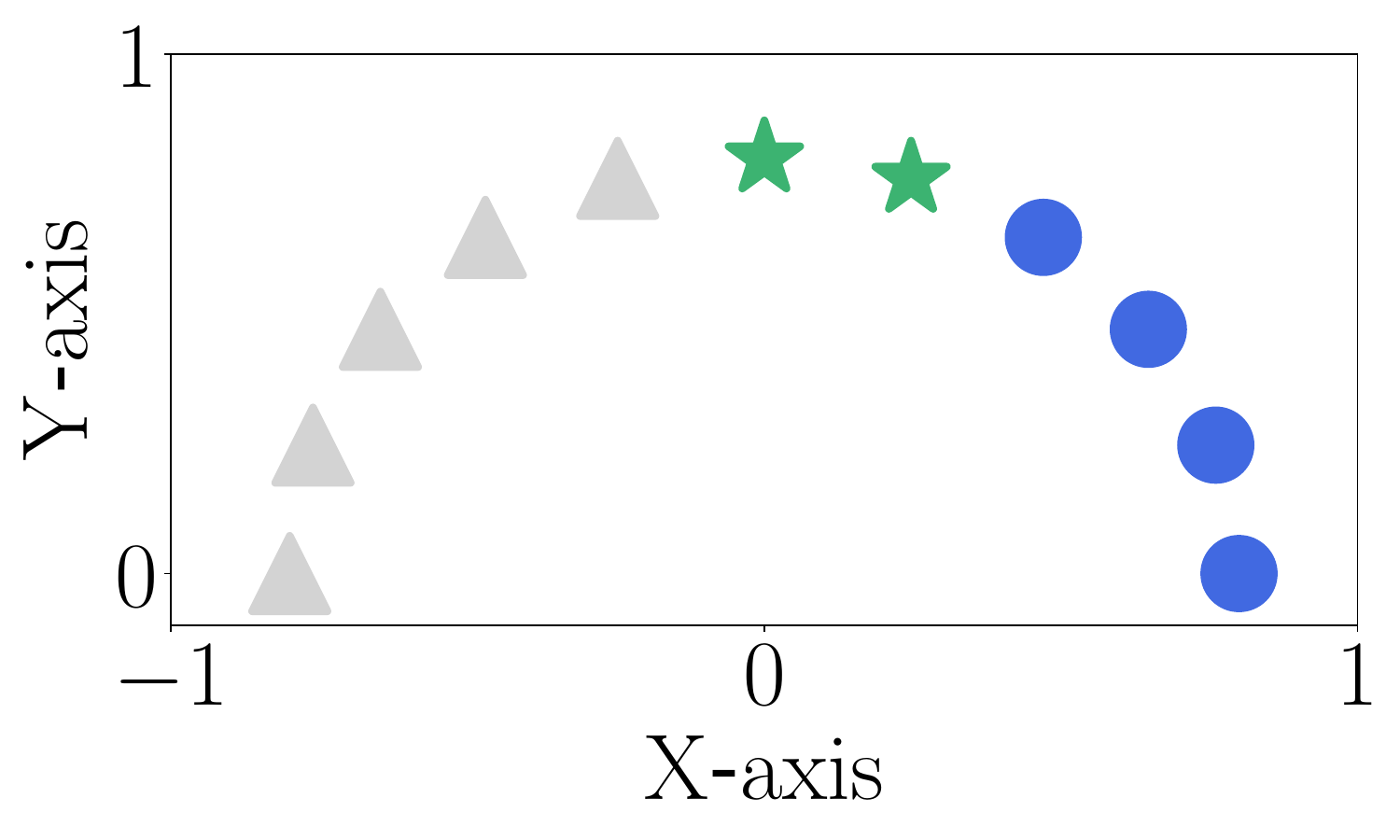}
        \caption{Partial sampling}\label{fig_navigation_partial}
    \end{subfigure}\hfill
    \begin{subfigure}[b]{0.23\textwidth}
        \centering
        \includegraphics[width=0.99\textwidth]{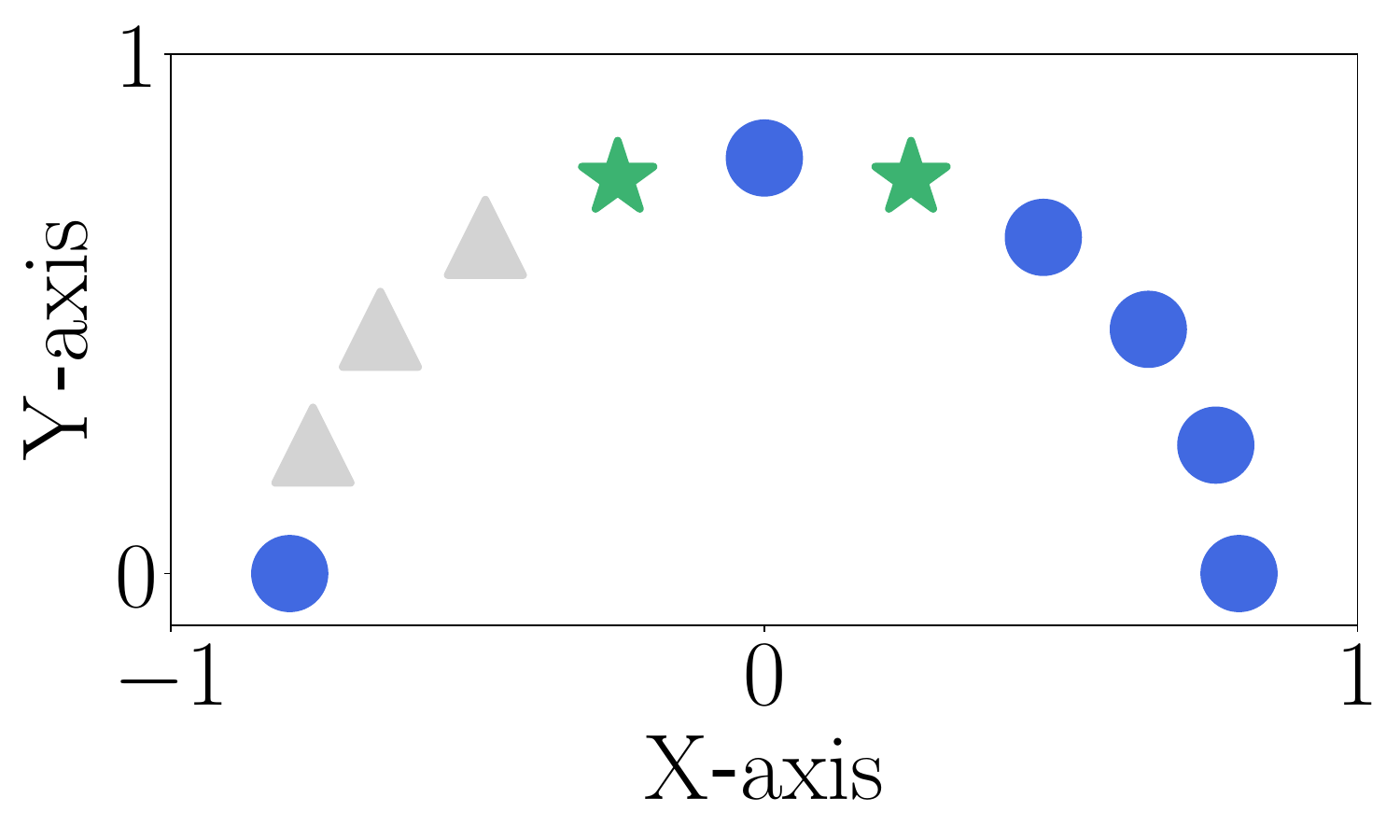}
        \caption{Biased sampling}\label{fig_navigation_biased}
    \end{subfigure}\hfill 
    \caption{We illustrate the meta-training curve and the corresponding meta task distribution of the navigation environment. The agent is meta-trained to navigate a training set of goals, and then tested on a distinct set of unseen test goals with a few adaptation steps.}\label{fig_navigation_results}
\end{figure}

\paragraph{Additional experiments.}
We find that the shifted distribution of the source task $\mathcal{T}$ affects the generalization on the unseen task.
We conduct a straightforward experiment in a 2D navigation environment to empirically investigate the influence of the source task distribution on meta reinforcement learning. 
In this environment, the agent starts at the origin $(0,0)$ and navigates towards a goal within a $[-1,1]\times[-1,1]$ plane. All goals are located on a circle centered at the origin with a fixed radius of $0.9$. An episode is considered a success if the agent's final distance to the goal is within a threshold of $0.1$. The agent is guided by a reward function defined as the negative Euclidean distance to the goal, with an additional bonus upon success.

For the experiment, we maintain a uniform distribution for the target tasks, where goals can appear anywhere on the circle. We then analyze the agent's adaptation performance after being meta-trained on source tasks drawn from three different distributions. The training curves are illustrated in Figure \ref{fig_navigation_curve}. We find that meta-training under uniform task distribution consistently outperforms that under partial and biased task distribution.
In detail, we find that:
\begin{enumerate}
    \item In Figure~\ref{fig_navigation_uniform}, when the source task distribution is uniform and covers the four directions of the target space, the agent can easily adapt to any target task.
    \item In Figure~\ref{fig_navigation_partial}, when the source tasks do not cover the target tasks, the policy overfits to the source distribution and fails to generalize.
    \item In Figure~\ref{fig_navigation_biased}, when the source distribution covers the target space but is heavily biased, the policy struggles to adapt to target tasks in the sparsely represented regions.
\end{enumerate}

\end{document}